\documentclass{article} % For LaTeX2e
\usepackage{iclr2027_conference,times}

\usepackage{amsmath,amsfonts,bm}

\def\eqref#1{equation~\ref{#1}}
\def\1{\bm{1}}

\DeclareMathAlphabet{\mathsfit}{\encodingdefault}{\sfdefault}{m}{sl}
\SetMathAlphabet{\mathsfit}{bold}{\encodingdefault}{\sfdefault}{bx}{n}

\usepackage{url}
\usepackage[english]{babel}

\usepackage{amsmath,amssymb,amsthm,mathtools}
\usepackage{graphicx}
\usepackage[protrusion=true,expansion=true]{microtype}
\usepackage{booktabs}
\newtheorem{proposition}{Proposition}

\usepackage[colorlinks=true, allcolors=blue]{hyperref}
\usepackage{subcaption}
\usepackage{wrapfig}
\usepackage{needspace}
\theoremstyle{plain}

\newtheorem*{proposition*}{Proposition}

\theoremstyle{remark}

\title{RepFlow: Reciprocal Supervision Improves Generation and Representation in Flow Models}

\author{
Weili Zeng\textsuperscript{1}\thanks{Email: \href{mailto:zwl666@sjtu.edu.cn}{zwl666@sjtu.edu.cn}.} \quad
Feng Tian\textsuperscript{1} \quad
Shengqi Liu\textsuperscript{1} \quad
Yichao Yan\textsuperscript{1}\thanks{Corresponding author: \href{mailto:yanyichao@sjtu.edu.cn}{yanyichao@sjtu.edu.cn}.} \\[0.5em]
\textsuperscript{1}Shanghai Jiao Tong University \\
}
\date{}

\iclrfinalcopy
\begin{document}
% Improve line breaking and discourage widows, orphans, and very short final lines.
\setlength{\emergencystretch}{1em}
\setlength{\parfillskip}{0pt plus 0.65\textwidth}
\clubpenalty=10000
\widowpenalty=10000
\displaywidowpenalty=10000

\maketitle
% Use a neutral arXiv preprint header rather than a conference-status header.
\pagestyle{plain}

\begin{abstract}
Generative models learn visual structure through denoising, yet their internal states are entangled with both noise level and network depth, making it difficult to obtain a stable visual representation from the generator itself. We introduce RepFlow, which learns such a representation from the generator's evolving computation and uses it to guide generation. Specifically, a separate timestep-free encoder is trained, through a timestep-conditioned predictor, to recover generator states across depths and noise levels from masked clean images. By excluding the reference noise and masked-out content from the encoder's input, we encourage the encoder to distill visual information that is recoverable from the visible image context and predictive of generator states. The representation learned from the generator's evolving states is then fed back to guide and improve the generator, whose updated states provide supervision for further representation learning, forming a reciprocal learning process. This reciprocal process improves multi-step generation and representation quality, as measured by frozen linear probing on ImageNet with latent-space SiT and pixel-space JiT, without an externally pretrained representation teacher. Across the two unconditional settings, FID decreases by $19.2$--$40.8\%$ relative to native training, while linear-probe accuracy improves by 6.80--10.04 percentage points over the best searched raw generator features. The learned representation also serves as a distributional metric for one-step JiT post-training, extending its role from instance-level alignment to distribution-level supervision.

\end{abstract}

\section{Introduction}
Generative modeling~\citep{ddpm,lipman2022flow,peebles2023scalable,sit} and representation learning~\citep{simclr,mae,ijepa} impose different demands on an image model. A generator must retain instance-specific detail to model the data distribution, whereas a representation learner is rewarded for discarding nuisance variation while preserving information useful for transfer. Both, however, must organize the same underlying visual structure. This shared requirement creates an opportunity to learn representations that are useful both for understanding images and for improving the process that generates them.

Prior work establishes that information can flow in both directions. Features extracted from pretrained generators support recognition and transfer~\citep{gabeur2026image,baranchuk2021label,tang2023emergent,stracke2025cleandift,ddae}, while pretrained visual representations can guide generative optimization and synthesis~\citep{repa,rae}. Self-representation alignment further shows that a generator can obtain useful guidance from its own internal states~\citep{sra}. We investigate how such guidance can be formed through a separate image representation learned alongside the generator. The question is how to turn the generator's evolving computation into a representation that can, in turn, guide its subsequent learning (Figure~\ref{intro_img2}).

\begin{figure}[t]
\centering
\includegraphics[width=\linewidth]{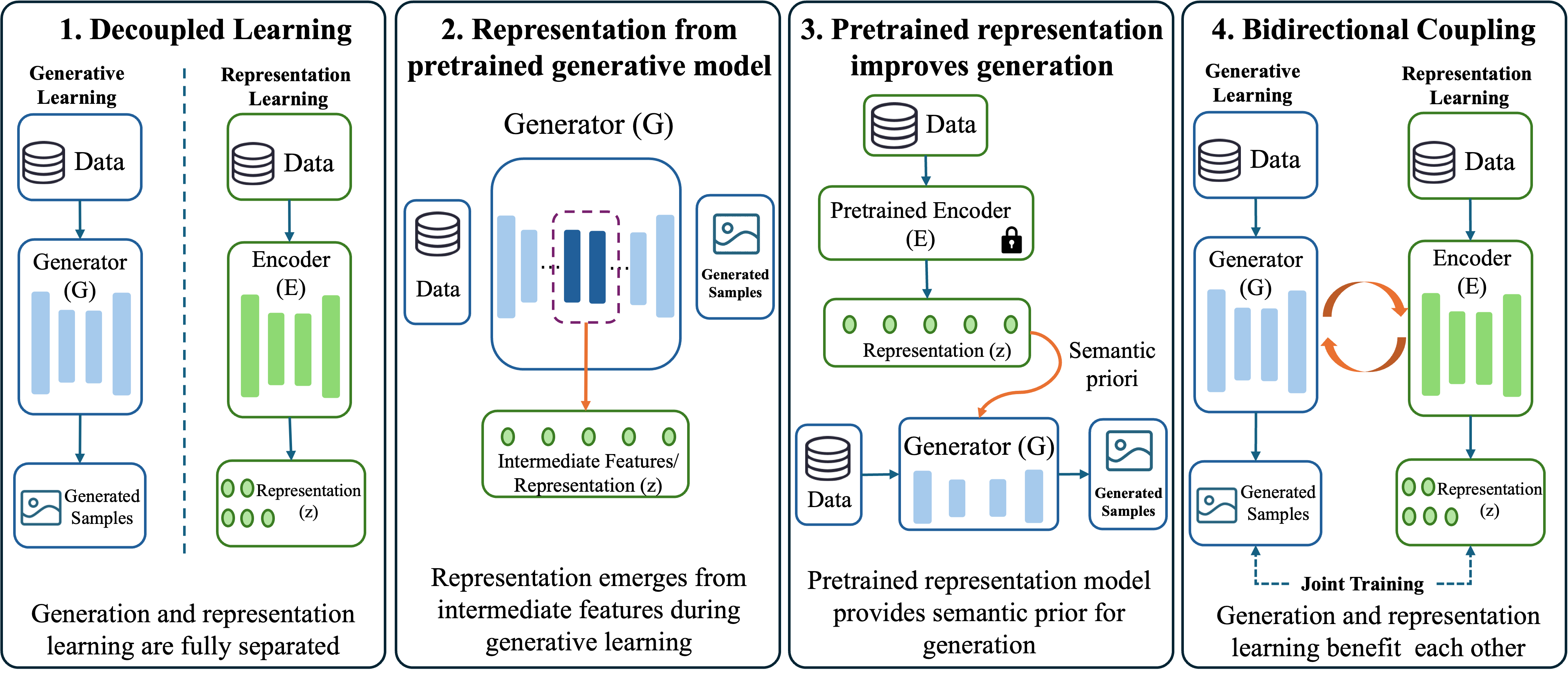}
\caption{Connections between generation and representation learning. Two common strategies extract features from a pretrained generator or use a pretrained representation to guide generation. RepFlow learns a clean-image representation from the generator's evolving states and uses it to shape subsequent generative learning.}
\label{intro_img2}
\end{figure}

The difficulty lies in the form of information available from a generator. A flow or diffusion model exposes a family of internal states indexed by noise level and network depth, each reflecting image content together with uncertainty and stage-specific denoising computation. A feature useful at one layer and timestep may behave differently elsewhere in the trajectory~\citep{baranchuk2021label,tang2023emergent,stracke2025cleandift,ddae}. These states nevertheless provide related observations of the same underlying image. Viewed as learning targets, their value depends on what an image encoder must learn to predict them. This perspective makes it possible to use the diversity of generative states as supervision for a single clean-image representation.

We introduce \textbf{RepFlow}, based on the idea that a generator can be guided by a representation learned from the predictable structure of its own computation. A separate student predicts generator states from several depths at a sampled timestep, given only a masked clean image. Since the student observes neither the reference noise nor the masked image content, reducing prediction error rewards information recoverable from visible context. Across training, the same encoder must support prediction across depths and noise levels, while timestep conditioning is confined to the predictor. The student thus learns through a family of contextual prediction tasks and produces a representation that can be extracted without selecting a generator layer or timestep.

This learned representation becomes part of the generator's own training signal. The student processes the full clean image and supplies an alignment target for a generator state. The returned target has been shaped by a different learning problem: predicting generative computation from restricted image context. It therefore provides a clean-image reference whose input does not depend on the current noise realization or denoising stage. As this reference guides the generator, the resulting updates change the states from which the student subsequently learns. RepFlow couples the two models through this evolving supervision, while retaining their separate parameters and the generator's native objective. The representation used to guide generation is learned within the same process, without an externally pretrained representation teacher.

On ImageNet, RepFlow improves both multi-step generation and frozen linear probing with latent-space SiT~\citep{sit} and pixel-space JiT~\citep{jit}. Predicting generator states already yields a student that outperforms the best searched raw generator features; reciprocal training further improves both the student and the generator. In unconditional experiments, RepFlow reduces FID from 32.83 to 19.43 on SiT and from 35.58 to 28.76 on JiT. Frozen linear-probe accuracy improves by 10.04 and 6.80 percentage points, respectively, over the best searched raw features of each generator. Class-conditional results are competitive with REPA~\citep{repa} and SRA~\citep{sra}. Controlled ablations favor generator-state targets over masked-pixel reconstruction and multi-depth supervision over a single-depth target, supporting the role of generative states as a source of predictive supervision.

The same representation also serves as a distributional metric for one-step JiT post-training~\citep{yang2026representation}: under matched initialization and optimization budgets, replacing a pretrained MAE feature with the RepFlow student lowers FID from 6.34 to 3.14. Its use thus extends from instance-level supervision during reciprocal learning to distribution-level supervision afterward.

Our contributions are:
\begin{itemize}
    \item We introduce RepFlow, a reciprocal learning framework in which the representation guiding a flow generator is learned from its own evolving computation through a separate predictive encoder, without an externally pretrained representation teacher.
    \item We show that stage-dependent generative states can supervise a timestep-free encoder whose features outperform the best searched raw generator features in linear probing.
    \item We demonstrate improvements in multi-step generation across latent-space SiT and pixel-space JiT, and show through matched one-step JiT post-training that the learned representation also provides effective distribution-level supervision.
\end{itemize}

\section{Preliminaries}

\subsection{Flow-based Generative Modeling}

Let $x_0 \sim p_{\mathrm{data}}$ denote a clean sample in the
generator's modeling space and $x_1 \sim p_{\mathrm{ref}}$ an
independent noise sample. In our experiments, $x_0$ consists of VAE
latents for SiT and image pixels for JiT. We define an interpolation
\begin{equation}
    x_t = \alpha_t x_0 + \sigma_t x_1,
    \qquad t \in [0,1],
\end{equation}
where $(\alpha_0,\sigma_0)=(1,0)$ and
$(\alpha_1,\sigma_1)=(0,1)$, so that $t=0$ corresponds to data and
$t=1$ to noise. The associated conditional velocity is
\begin{equation}
    u_t(x_0,x_1)
    =
    \frac{\mathrm{d}x_t}{\mathrm{d}t}
    =
    \dot{\alpha}_t x_0+\dot{\sigma}_t x_1.
\end{equation}

SiT~\citep{sit} directly parameterizes a velocity field
$v_\theta(x_t,t)$ and minimizes
\begin{equation}
    \mathcal{L}_{\mathrm{FM}}
    =
    \mathbb{E}_{x_0,x_1,t}
    \left[
        \left\|v_\theta(x_t,t)-u_t(x_0,x_1)\right\|_2^2
    \right].
\end{equation}

JiT~\citep{jit} directly predicts clean data,
$\hat{x}_{0,\theta}=G_\theta(x_t,t)$. For the linear path
$\alpha_t=1-t$ and $\sigma_t=t$ used in our experiments, its native
velocity-space loss can be written as a weighted reconstruction
objective:
\begin{equation}
    \mathcal{L}_{\mathrm{JiT}}
    =
    \mathbb{E}_{x_0,x_1,t}
    \left[
        w(t)
        \left\|G_\theta(x_t,t)-x_0\right\|_2^2
    \right],
\end{equation}
where $w(t)=\max(t,t_{\min})^{-2}$ and $t_{\min}$ is the
denominator-clipping threshold specified in Appendix~\ref{app:exp}.
Expectations use each backbone's native timestep sampling
distribution. We write $G_\theta$ for either generative backbone,
$\mathcal{L}_{\mathrm{gen}}$ for its native objective, and
$h_\ell^G(x_t,t)$ for its internal state at layer $\ell$.

\subsection{Predictive Representation Learning}

Masked predictive learning trains an encoder by predicting targets
from partial image observations~\citep{data2vec,ijepa,mae}. Let $x$ denote
the clean RGB image corresponding to $x_0$, and let
$x^{\mathrm{mask}}$ be a masked view. The encoder $U_\phi$ produces
\begin{equation}
    u^{\mathrm{mask}}(x)
    =
    U_\phi\!\left(x^{\mathrm{mask}}\right).
\end{equation}

Given a target feature $y$ and an optional conditioning variable $c$,
a predictor $P_\psi$ is trained through
\begin{equation}
    \mathcal{L}_{\mathrm{pred}}
    =
    \mathbb{E}
    \left[
        \ell\!\left(
            P_\psi\!\left(u^{\mathrm{mask}}(x),c\right),
            \operatorname{sg}(y)
        \right)
    \right],
\end{equation}
where $\ell$ is a prediction loss and $\operatorname{sg}(\cdot)$
stops gradients through the target in this objective. The
target-producing model may still be updated by other objectives.
In RepFlow, $y$ is constructed from generator states and $c$ is the
generative timestep, supplied only to the predictor. This keeps the
encoder's input independent of the target's generative stage.

\section{Method}
\label{sec:method}

\subsection{From Generative States to Predictive Targets}
\label{sec:method:overview}

RepFlow learns the representation that guides a generator from the
predictable structure of that generator's own computation. Internal
states at different depths and noise levels define related prediction
tasks about the same image (Figure~\ref{intro_img1}). We train a
separate image encoder through these tasks, then use its full-image
features to guide the generator. The resulting updates change the
states that supervise subsequent representation learning.

\begin{figure}[t]
\centering
\includegraphics[width=\linewidth]{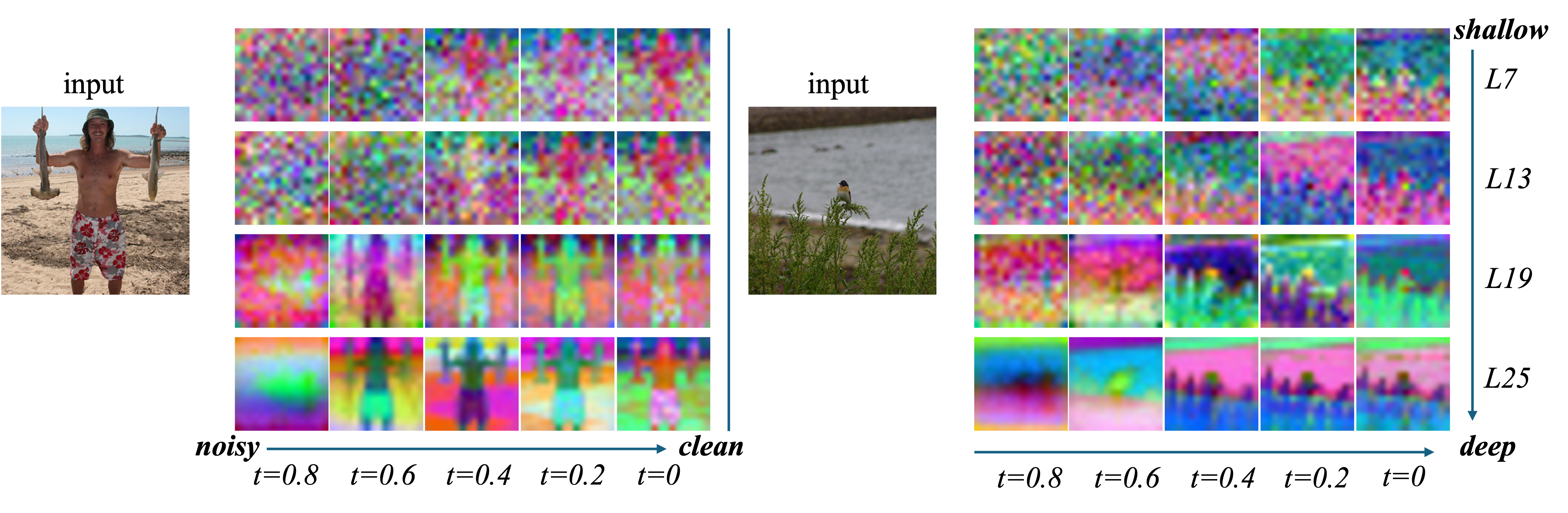}
\caption{Generator states across network depth and noise level.
The same image induces different feature maps at different stages
of generative computation. RepFlow uses these states as prediction
targets for a clean-image encoder.}
\label{intro_img1}
\end{figure}

The predictive task restricts what the encoder can observe.
Let $\mathcal{M}$ denote the masked spatial positions, and let
$x^{\mathrm{mask}}$ contain the visible image patches and their
positions. For a fixed generator, let $Y_{t,i}$ collect normalized
generator features from selected depths at a masked position $i$.
Under a squared-error surrogate, the Bayes-optimal prediction
satisfies
\begin{equation}
    F_i^\star(x^{\mathrm{mask}},t)
    =
    \mathbb{E}\!\left[
        Y_{t,i}\mid x^{\mathrm{mask}},t
    \right],
    \qquad i\in\mathcal{M}.
\end{equation}
Visible context is therefore useful insofar as it explains target
variation beyond the timestep alone. Only the variation unresolved
by this conditioning contributes to irreducible squared prediction
error. Supplying $t$ to the predictor allows prediction to vary
with the generative stage while the encoder depends only on image
observations. Appendix~\ref{sec:theory} gives the corresponding
risk decomposition. This squared-loss analysis explains the
prediction problem; the implemented objective uses Smooth L1,
as specified below.

\begin{figure}[t]
\centering
\includegraphics[width=\linewidth]{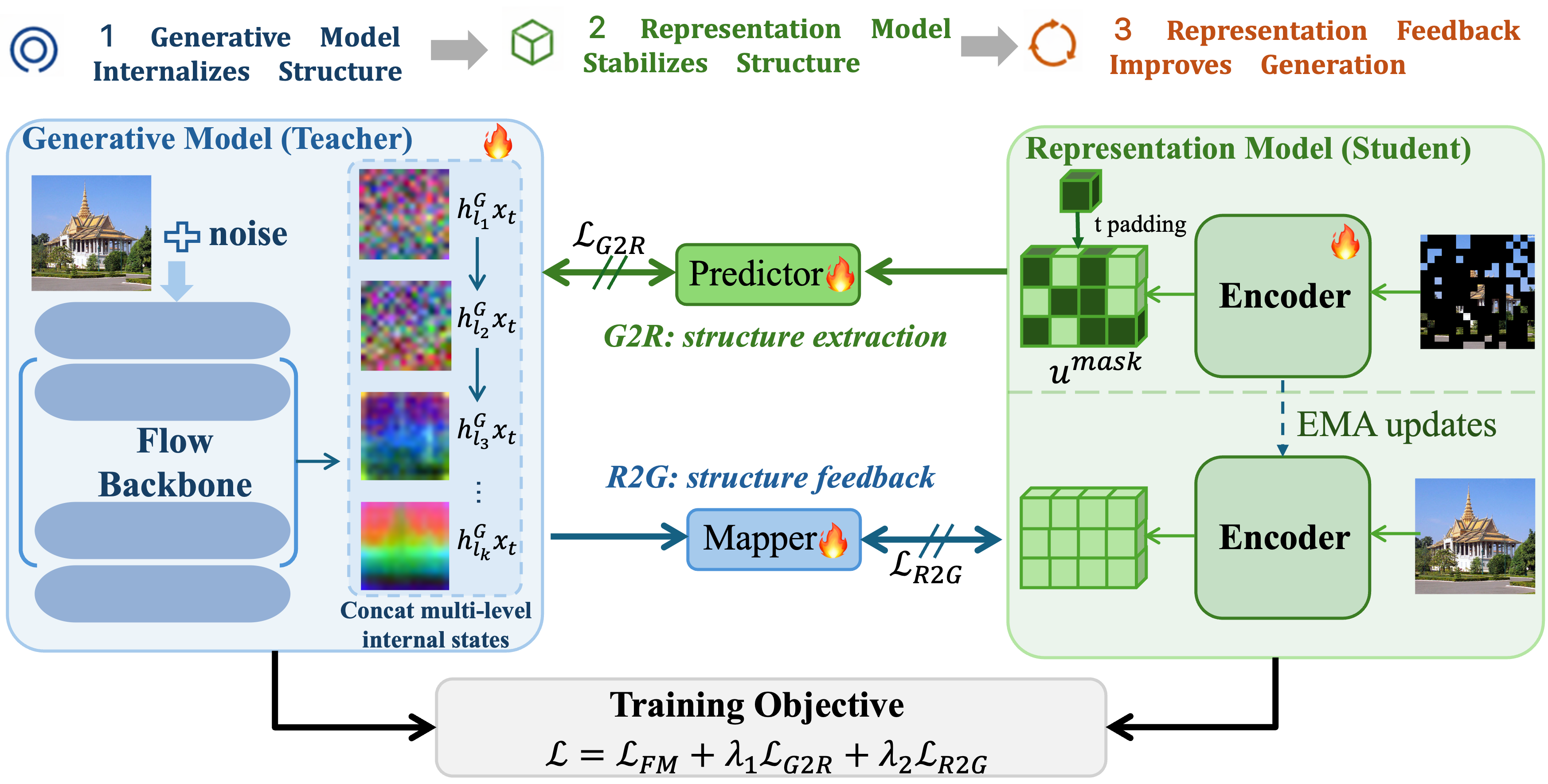}
\caption{RepFlow. Generator states supervise masked prediction
through G2R, while full-image features from an EMA encoder guide
an early generator state through R2G. Stopped targets separate the
gradient paths, and the models interact through the targets used
in subsequent updates. The native generator objective is
$\mathcal{L}_{\mathrm{FM}}$ for SiT and
$\mathcal{L}_{\mathrm{JiT}}$ for JiT. Flames indicate trainable
modules; double slashes indicate stopped gradients.}
\label{framework}
\end{figure}

\subsection{Learning the Supervisory Representation}

The generator $G_\theta$ and encoder $U_\phi$ have separate
parameters. We use $x$ for the clean RGB image and $x_0$ for its
representation in the generator's modeling space: a VAE latent
for SiT and pixels for JiT. The generator receives the
corresponding noised input $x_t$, while the encoder receives a
masked view of $x$. In both instantiations, generator features
and full-image encoder features share an $N$-position spatial grid.

Let $\mathcal{D}=\{\ell_1,\ldots,\ell_K\}$ be the selected
generator depths. At each spatial position $i$, the predictive
target is
\begin{equation}
    Y_{t,i}
    =
    \operatorname{Concat}_{\ell\in\mathcal{D}}
    \left[
        \operatorname{LN}
        \left(h_{\ell,i}^G(x_t,t)\right)
    \right],
\end{equation}
where $\operatorname{LN}$ denotes channel-wise layer normalization
at each token. Normalization places features from different depths
on a common scale before concatenation. Class conditioning, when
present, is implicit in the generator states; the encoder receives
no class label, reference noise, or timestep.

The encoder produces a masked-view code
$u^{\mathrm{mask}}_\phi(x)$. A timestep-conditioned predictor
reconstructs a full spatial grid of target predictions,
$\hat{Y}_t=P_\psi(u^{\mathrm{mask}}_\phi(x),t)$.
The encoder features supplied to this predictor are specified for
each setting in Appendix~\ref{app:exp}. We supervise only masked
positions:
\begin{equation}
    \mathcal{L}_{\mathrm{G2R}}
    =
    \mathbb{E}\!\left[
        \frac{1}{|\mathcal{M}|}
        \sum_{i\in\mathcal{M}}
        \ell_{\mathrm{SL1}}
        \left(
            \hat{Y}_{t,i},
            \operatorname{sg}(Y_{t,i})
        \right)
    \right],
\end{equation}
where $\ell_{\mathrm{SL1}}$ is Smooth L1 with $\beta=1$, averaged
over feature channels. Expectations include the sampled data,
reference noise, timestep, and mask. The resulting loss updates
only the encoder and predictor.

Across training, the encoder supports prediction at several
depths and sampled noise levels from image context alone.
The predictor handles timestep-dependent decoding, while the
encoder learns the information useful across these prediction
tasks. The generator continues to evolve under the training
objective below, so the student's targets evolve with it.

\subsection{Reciprocal Training}

The learned representation supplies an auxiliary target for
the generator (Figure~\ref{framework}). An EMA encoder
$U_{\bar\phi}$ processes the full clean image. For each position
$i$, we define its feature and the projected generator feature as
\begin{equation}
    V_i(x)=\left[U_{\bar\phi}(x)\right]_i,
    \qquad
    Z_i(x_t,t)
    =
    g_\omega\!\left(h_{\ell_s,i}^G(x_t,t)\right),
\end{equation}
where $\ell_s$ is an early generator layer and $g_\omega$ maps
its features to the encoder's feature dimension. Placing the
alignment at an early layer allows the resulting updates to
affect the states processed by subsequent generator blocks.

We use token-wise cosine alignment~\citep{repa}, with the
supervisory target learned through the predictive task above.
Writing $\nu(a)=a/\|a\|_2$ for feature normalization,
\begin{equation}
    \mathcal{L}_{\mathrm{R2G}}
    =
    \mathbb{E}\!\left[
        \frac{1}{N}\sum_{i=1}^{N}
        \left(
            1-
            \left\langle
                \nu(Z_i),
                \operatorname{sg}\!\left(\nu(V_i)\right)
            \right\rangle
        \right)
    \right].
\end{equation}
The target is evaluated on the full image and is independent
of the current noise realization and timestep. Its parameters
are nevertheless shaped by predicting generator states from
masked context.

The complete training objective is
\begin{equation}
    \mathcal{L}_{\mathrm{joint}}
    =
    \mathcal{L}_{\mathrm{gen}}
    +
    \lambda_1\mathcal{L}_{\mathrm{G2R}}
    +
    \lambda_2\mathcal{L}_{\mathrm{R2G}},
\end{equation}
where $\mathcal{L}_{\mathrm{gen}}$ is the native SiT or JiT
objective. The stopped targets give
\begin{equation}
    \begin{aligned}
    \nabla_\theta\mathcal{L}_{\mathrm{joint}}
    &=
    \nabla_\theta\mathcal{L}_{\mathrm{gen}}
    +
    \lambda_2\nabla_\theta\mathcal{L}_{\mathrm{R2G}},\qquad
    \nabla_\phi\mathcal{L}_{\mathrm{joint}}
    =
    \lambda_1\nabla_\phi\mathcal{L}_{\mathrm{G2R}}.
    \end{aligned}
\end{equation}
The predictor and projection head are optimized through G2R
and R2G, respectively. After each online update, we update
$\bar\phi\leftarrow m\bar\phi+(1-m)\phi$.
We use $\lambda_1=1$, $\lambda_2=0.1$, and $m=0.9999$;
architectural and optimization settings are given in
Appendix~\ref{app:exp}.

Reciprocity arises through the supervision available at the
next update. Updating $\theta$ changes the states that the
student must predict; updating $\phi$ changes the EMA features
that guide subsequent generator updates. The encoder thereby
turns the generator's computation into a clean-image training
target, while its own prediction task changes as the generator
learns. Appendix~\ref{sec:theory} formalizes these update
dependencies.

For representation evaluation, we average the final-layer tokens
from $U_{\bar\phi}(x)$. This requires a single full-image encoder
pass, with no generator, predictor, or layer--timestep selection.

\subsection{Representation-guided One-step Post-training}

Starting from a jointly trained JiT checkpoint, we freeze the
EMA encoder and reuse its feature space for one-step distribution
matching. A terminal-noise input produces
$\hat{x}=G_\theta(x_1,y,1)$, where $y$ denotes the class condition.
Following FD-loss~\citep{yang2026representation}, we define
$r(x)=\operatorname{Pool}(U_{\bar\phi}(x))$
and compute the real feature mean and covariance
$(\mu_r,\Sigma_r)$ once from $\mathcal{D}_{\mathrm{train}}$.
Generated statistics $(\mu_g,\Sigma_g)$ are tracked through
an EMA of first and second moments. The objective is
\begin{equation}
    \begin{aligned}
    \mathcal{L}_{\mathrm{FD}}
    &=
    \|\mu_r-\mu_g\|_2^2
    +d_{\mathrm{cov}}(\Sigma_r,\Sigma_g),\\
    d_{\mathrm{cov}}(\Sigma_r,\Sigma_g)
    &=
    \operatorname{Tr}\!\left(
        \Sigma_r+\Sigma_g
        -2\left(
            \Sigma_r^{1/2}\Sigma_g\Sigma_r^{1/2}
        \right)^{1/2}
    \right).
    \end{aligned}
\end{equation}
Real statistics and historical generated moments remain fixed
during backpropagation; gradients through the current
generated-image features update the generator while the encoder
parameters remain frozen. This reuses the representation learned
by reciprocal training as a distributional metric, extending its
role beyond instance-level alignment.

\section{Experiments}

Our experiments test whether a representation learned from a generator can improve the generator that supplies its training targets. We first evaluate sample quality and frozen representations under reciprocal training, then use controlled ablations to examine how feedback and target choice affect both. Class-conditional generation and one-step post-training test the utility of the learned supervision in the presence of class labels and under a different training objective.

\subsection{Experimental Setup}

We instantiate RepFlow with latent-space SiT~\citep{sit} and pixel-space JiT~\citep{jit} on ImageNet~\citep{imagenet}. Each native/RepFlow comparison matches the generative backbone, data, epoch budget, sampler, and evaluator. We report FID over 50,000 generated samples (FID-50K)~\citep{fid} and ImageNet top-1 accuracy from frozen linear probing. All of our generation evaluations use $\mathrm{CFG}=1$.

RepFlow-B and RepFlow-L denote the use of a separate ViT-B or ViT-L student~\citep{vit}, respectively; the generator backbone is unchanged. For the raw-generator representation baseline, we search layers and timesteps and select block 21 at implementation timestep $\tau=0.7$ for SiT and $\tau=0.5$ for JiT, where $\tau=1-t$. RepFlow uses the final-layer, full-view EMA student feature in a single clean-image forward pass.

Class-conditional comparisons include REPA~\citep{repa}, which uses a pretrained DINOv2 representation target, and SRA~\citep{sra}, which aligns states within the generator. Appendix~\ref{app:exp} details the architectures, optimization, evaluation protocols, and computational cost. Training progress below is compared in epochs; the student adds computational cost per update.

\subsection{Unconditional Generation and Representation}

RepFlow improves generation on both backbones (Figure~\ref{fig:uncond}). After 600 epochs, RepFlow-L reduces SiT FID from 32.83 to 19.43 (40.8\%) and JiT FID from 35.58 to 28.76 (19.2\%). RepFlow-B reaches 20.12 and 28.80, respectively, so most of the generative gain is already present with the smaller student. The benefit also appears before the final checkpoint: RepFlow-L reaches the native 600-epoch FID after approximately 143 epochs on SiT and 114 epochs on JiT.

The same training runs yield stronger frozen representations (Figure~\ref{fig:uncond_representation}). On SiT, linear-probe accuracy increases from 58.02\% for the best searched raw generator feature to 64.70\% with RepFlow-B and 68.06\% with RepFlow-L. On JiT, the corresponding accuracies are 61.79\%, 67.00\%, and 68.59\%. RepFlow-L therefore improves on the best searched raw features by 10.04 and 6.80 percentage points while requiring no layer--timestep selection at inference. Increasing student capacity improves linear probing on both backbones, although JiT FID changes little between the two student sizes. Recognition accuracy and effectiveness as generative supervision thus need not improve to the same extent.

\begin{figure}[t]
\centering
\begin{minipage}[t]{0.49\linewidth}
    \vspace{0pt}\centering
    \includegraphics[width=\linewidth]{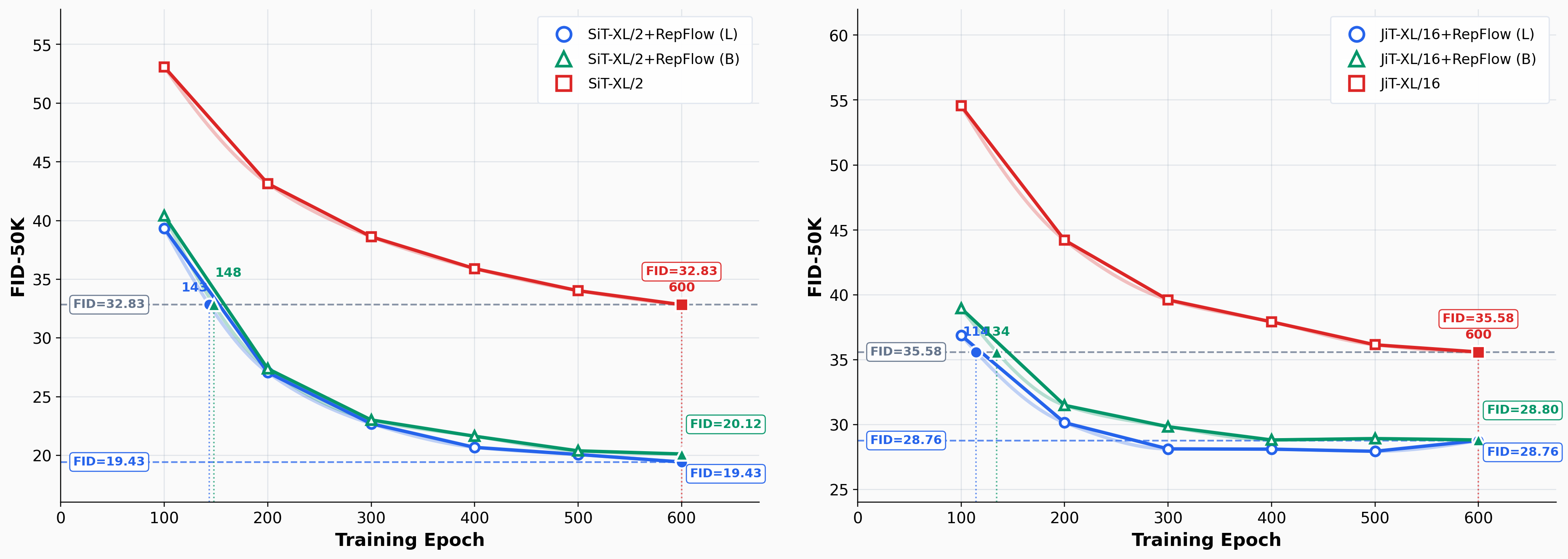}
    \vspace{-1mm}
    \makebox[\linewidth]{\hspace{0.12\linewidth}\textbf{(a) SiT}\hfill\textbf{(b) JiT}\hspace{0.12\linewidth}}
    \caption{Unconditional ImageNet FID-50K across training epochs. Annotations mark when RepFlow reaches the native 600-epoch FID.}
    \label{fig:uncond}
\end{minipage}
\hfill
\begin{minipage}[t]{0.49\linewidth}
    \vspace{0pt}\centering
    \includegraphics[width=0.49\linewidth]{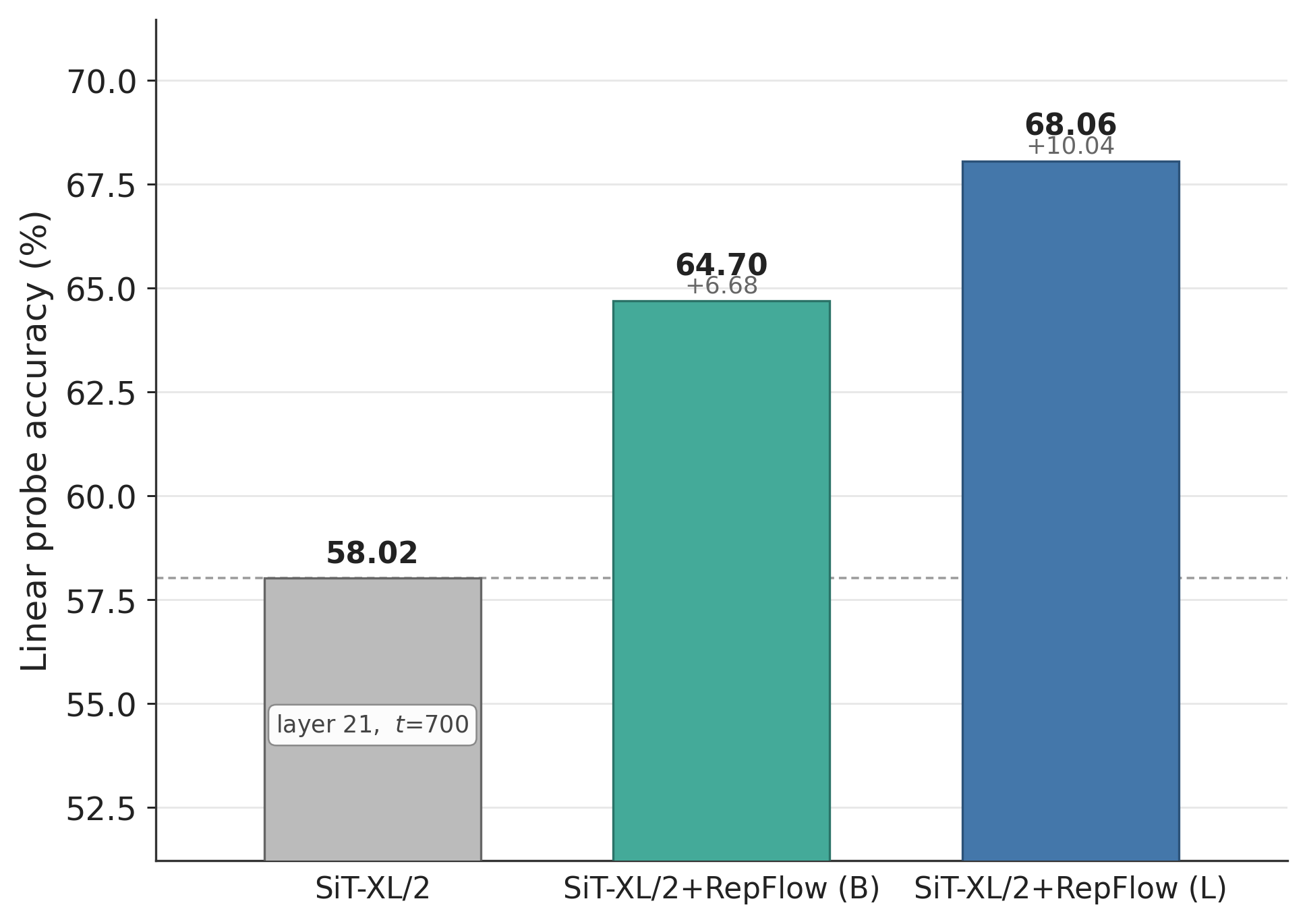}\hfill
    \includegraphics[width=0.49\linewidth]{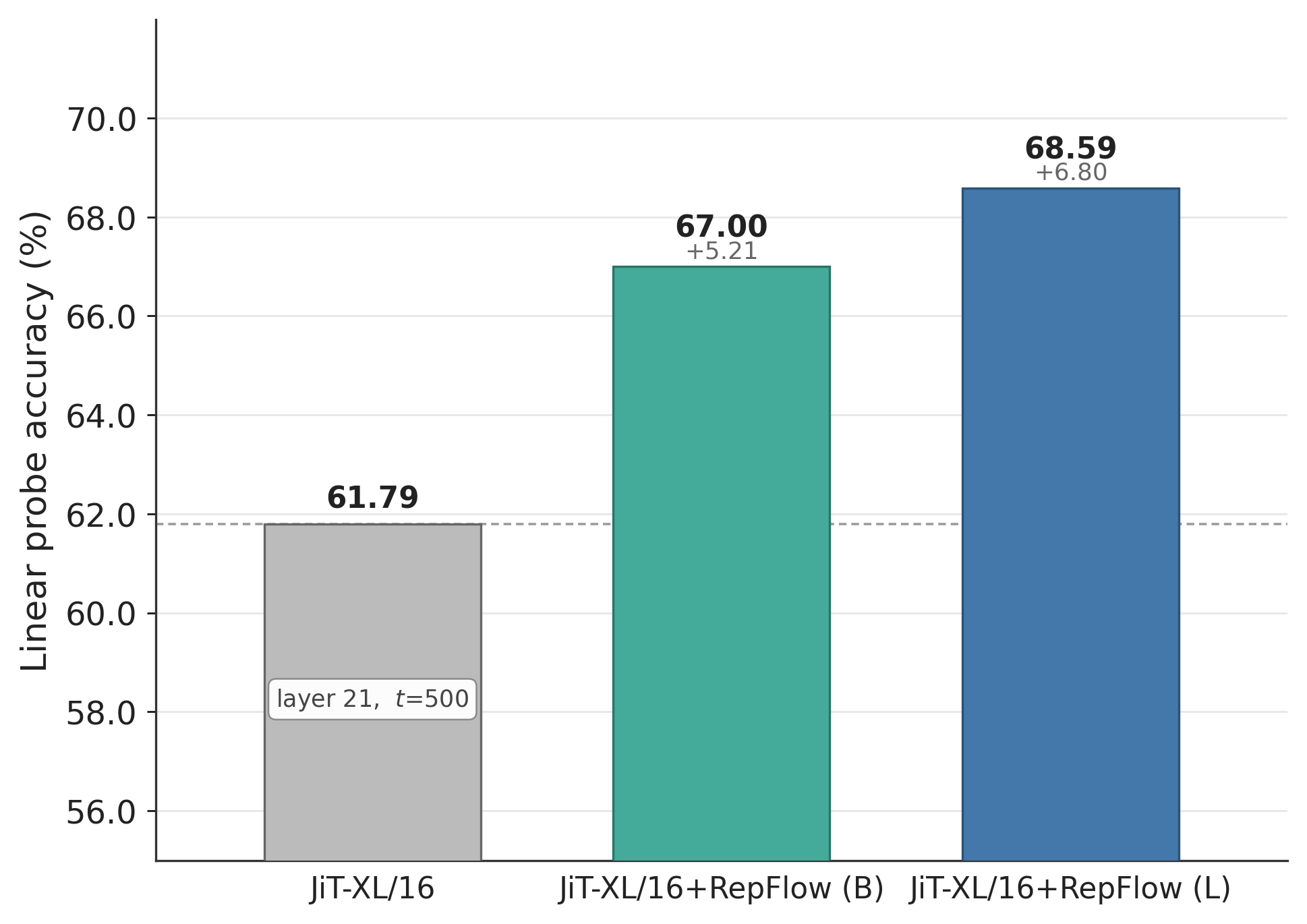}
    \vspace{-1mm}
    \makebox[\linewidth]{\hspace{0.12\linewidth}\textbf{(a) SiT}\hfill\textbf{(b) JiT}\hspace{0.12\linewidth}}
    \caption{Frozen ImageNet linear probing. Raw generator features use the best searched layer and timestep; RepFlow uses the final-layer, full-view student feature.}
    \label{fig:uncond_representation}
\end{minipage}
\vspace{-2mm}
\end{figure}

\subsection{Analyzing Reciprocal Supervision}

\begin{wraptable}{r}{0.55\textwidth}
\vspace{-0.8\baselineskip}
\centering
\caption{Unconditional ImageNet ablations at 600 epochs. LP denotes linear-probe accuracy (\%).}
\label{tab:reciprocal_ablations}
\scriptsize
\setlength{\tabcolsep}{2.3pt}
\renewcommand{\arraystretch}{0.96}
\begin{tabular}{llcccc}
\toprule
& & \multicolumn{2}{c}{SiT} & \multicolumn{2}{c}{JiT} \\
\cmidrule(lr){3-4}\cmidrule(lr){5-6}
Component & Variant & FID $\downarrow$ & LP $\uparrow$ & FID $\downarrow$ & LP $\uparrow$ \\
\midrule
Reciprocal path & Native / raw state & 32.83 & 58.02 & 35.58 & 61.79 \\
& G2R only & 32.83 & 65.35 & 35.58 & 66.19 \\
& G2R+R2G & \textbf{19.43} & \textbf{68.06} & \textbf{28.76} & \textbf{68.59} \\
\midrule
G2R target & Pixel reconstruction & 25.23 & 57.23 & 31.56 & 59.88 \\
& Generator states & \textbf{19.43} & \textbf{68.06} & \textbf{28.76} & \textbf{68.59} \\
\midrule
Target depth & Single level & 20.71 & 65.26 & 30.44 & 67.17 \\
& Multi-level & \textbf{19.43} & \textbf{68.06} & \textbf{28.76} & \textbf{68.59} \\
\bottomrule
\end{tabular}
\vspace{-0.8\baselineskip}
\end{wraptable}

Table~\ref{tab:reciprocal_ablations} evaluates all student variants with ViT-L. For linear probing, Native uses its best searched raw generator feature, whereas the remaining variants use the frozen student. The single-level targets are block 19 for SiT and block 20 for JiT; the multi-level target set is $\mathcal{D}=\{5,12,19,26\}$.

Table~\ref{tab:reciprocal_ablations} first isolates the effect of returning the learned representation to the generator. Disabling R2G ($\lambda_2=0$) leaves the generator's native update unchanged because G2R uses stopped generator targets. Under matched initialization and stochasticity, Native and G2R-only therefore share the same generator trajectory. Even in this setting, masked prediction raises linear-probe accuracy from 58.02\% to 65.35\% on SiT and from 61.79\% to 66.19\% on JiT. Learning from the generator's evolving states thus produces a stronger representation than reading out its best searched raw feature.

Enabling feedback improves both models: FID decreases by 13.40 on SiT and 6.82 on JiT, while student accuracy increases by a further 2.71 and 2.40 percentage points. R2G supplies no direct gradient to the student, so its influence on representation learning passes through the generator states used as G2R targets. The additional student gains are therefore consistent with feedback improving the supervision available for subsequent representation learning.

Target choice also affects both outcomes. We replace generator-state prediction with MAE-style masked-pixel reconstruction while retaining the student, masking ratio, R2G branch, data, and training budget. Pixel reconstruction itself improves native FID, reaching 25.23 on SiT and 31.56 on JiT. Predicting generator states further reduces these values to 19.43 and 28.76, while raising student accuracy from 57.23\% to 68.06\% and from 59.88\% to 68.59\%, respectively. This control shows that the target used to train the student affects both the representation it acquires and the quality of the generative supervision it supplies.

Supervision across depth provides a further gain. Relative to the strongest evaluated single-depth target, multi-level supervision lowers FID by 1.28 on SiT and 1.68 on JiT and improves linear-probe accuracy by 2.80 and 1.42 percentage points. These results favor learning from several stages of generator computation in the tested settings, while retaining a single student representation at inference.

\subsection{Class-Conditional Generation}

\begin{wraptable}{r}{0.52\textwidth}
\vspace{-0.8\baselineskip}
\centering
\caption{Class-conditional ImageNet at 600 epochs. Bold/underline denote the best/second-best results per backbone; ties share formatting.}
\label{tab:conditional_600}
\scriptsize
\setlength{\tabcolsep}{2.8pt}
\renewcommand{\arraystretch}{0.96}
\begin{tabular}{lccccc}
\toprule
Method & FID$\downarrow$ & sFID$\downarrow$ & IS$\uparrow$ & Pre.$\uparrow$ & Rec.$\uparrow$ \\
\midrule
SiT & 9.62 & 5.31 & 123.0 & 0.67 & 0.68 \\
SiT + SRA & \underline{6.77} & \textbf{4.91} & 142.0 & \underline{0.69} & 0.68 \\
SiT + RepFlow & 6.99 & 5.17 & \underline{147.9} & 0.68 & \textbf{0.70} \\
SiT + REPA & \textbf{6.02} & \underline{5.05} & \textbf{159.3} & \textbf{0.70} & \underline{0.69} \\
\midrule
JiT & 16.80 & \underline{6.93} & 93.7 & 0.61 & \underline{0.67} \\
JiT + SRA & 14.24 & 7.04 & 101.8 & \underline{0.63} & \underline{0.67} \\
JiT + RepFlow & \underline{12.65} & \textbf{6.52} & \underline{116.4} & \underline{0.63} & \textbf{0.68} \\
JiT + REPA & \textbf{12.01} & 7.93 & \textbf{122.0} & \textbf{0.64} & \underline{0.67} \\
\bottomrule
\end{tabular}
\vspace{-0.8\baselineskip}
\end{wraptable}

RepFlow also improves generation when class labels already provide semantic conditioning (Table~\ref{tab:conditional_600} and Figure~\ref{fig:fid_comparison}). Table~\ref{tab:conditional_600} reports FID, spatial FID (sFID), IS, precision (Pre.), and recall (Rec.). RepFlow and SRA use no externally pretrained representation teacher, whereas REPA uses pretrained DINOv2. At 600 epochs, RepFlow reduces SiT FID from 9.62 to 6.99 (27.3\%) and JiT FID from 16.80 to 12.65 (24.7\%). The learned image-level supervision therefore remains useful in the presence of class conditioning.

Against SRA, RepFlow has slightly higher SiT FID (6.99 versus 6.77), with higher IS and recall; on JiT, it achieves lower FID and sFID and higher IS and recall. REPA retains lower FID on both backbones, with gaps of 0.97 on SiT and 0.64 on JiT, while RepFlow achieves higher recall. These comparisons place the representation learned within generative training close to pretrained-encoder guidance in FID, with different tradeoffs across the other metrics.

\begin{figure}[!ht]
\centering
\begin{subfigure}{0.48\linewidth}
    \centering
    \includegraphics[width=0.9\linewidth]{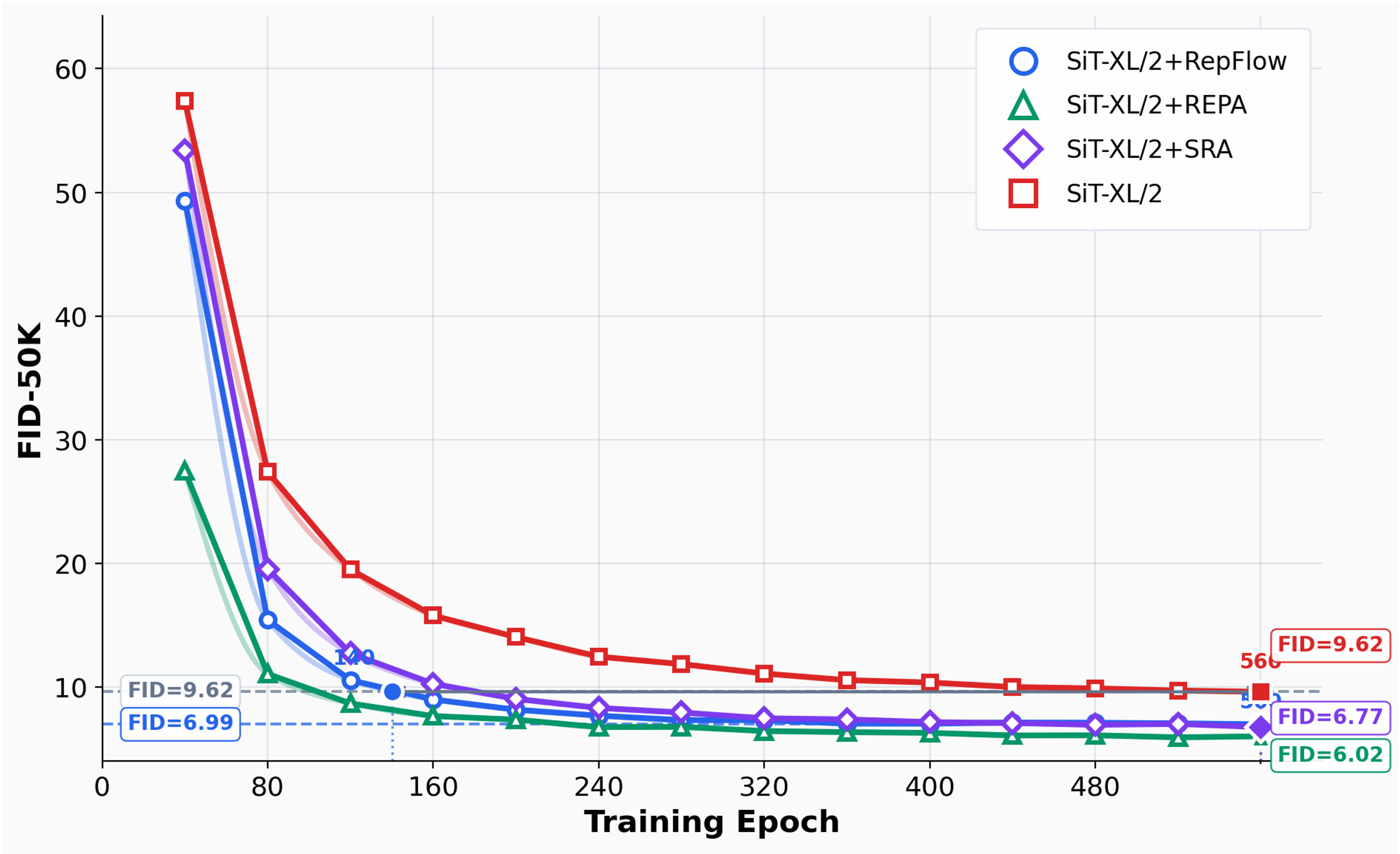}
    \caption{SiT}
    \label{fig:sit_fid}
\end{subfigure}
\hfill
\begin{subfigure}{0.48\linewidth}
    \centering
    \includegraphics[width=\linewidth]{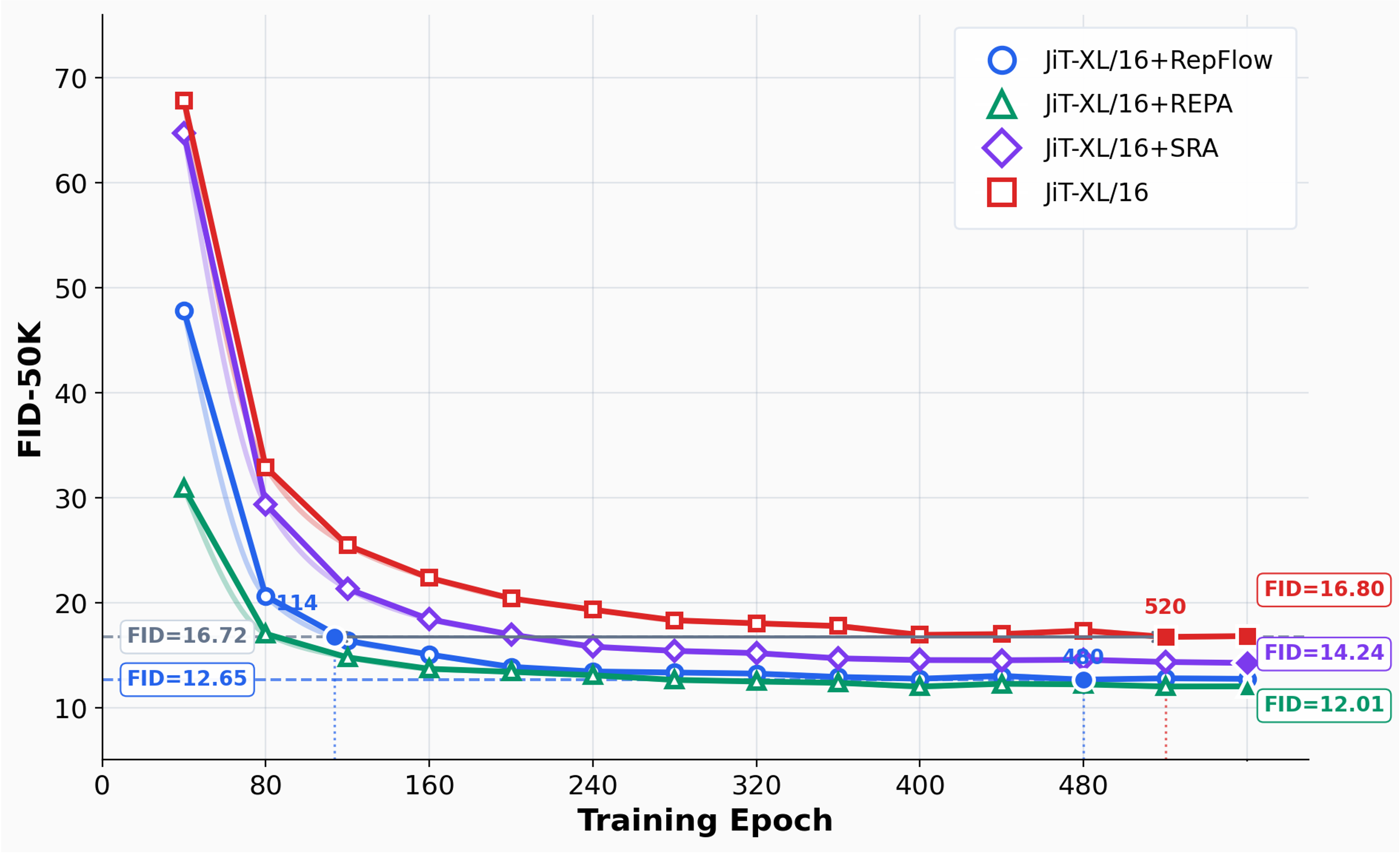}
    \caption{JiT}
    \label{fig:jit_fid}
\end{subfigure}
\caption{Class-conditional ImageNet FID-50K across training epochs for native training, RepFlow, SRA, and REPA. RepFlow learns its representation target concurrently with the generator.}
\label{fig:fid_comparison}
\vspace{-2mm}
\end{figure}

\subsection{Representation-Guided One-Step Post-Training}

We next test whether a representation learned through instance-level reciprocal supervision can also guide distribution matching after joint training. Starting from the same 400-epoch RepFlow-JiT checkpoint, we freeze either the RepFlow ViT-L/16 student or an ImageNet-pretrained MAE ViT-L/16 encoder, matching the student architecture, and update only the generator for 30 epochs using FD-loss~\citep{yang2026representation}, as described in Section~\ref{sec:method}. The data, sampling protocol, optimizer, schedule, and $\mathrm{CFG}=1$ are shared; only the feature space changes.

\begin{table}[!ht]
\centering
\caption{Class-conditional one-step ImageNet generation. The final two rows share the generator checkpoint and post-training protocol. Ext. Enc. lists externally pretrained representation encoders. $\mathrm{FDr}^{6}$~\citep{yang2026representation} averages normalized Fr\'echet distances in Inception~\citep{inception}, ConvNeXt~\citep{convnet}, DINOv2~\citep{dinov2}, MAE~\citep{mae}, SigLIP~\citep{siglip}, and CLIP~\citep{clip} spaces. $\dagger$ denotes FD-SIM training in SigLIP, Inception, and MAE spaces. RepFlow-student FD uses its own learned encoder.}
\label{tab:onestep_system}
\resizebox{\textwidth}{!}{
\begin{tabular}{lccccccccc}
\toprule
Method & Space & Epoch & NFE & Ext. Enc. & FID $\downarrow$ & $\mathrm{FDr}^{6}$ $\downarrow$ & IS $\uparrow$ & Prec. $\uparrow$ & Rec. $\uparrow$ \\
\midrule
\multicolumn{10}{l}{\emph{One-step generative models}} \\
iCT-XL/2~\citep{icm} & latent & -- & 1 & No & 34.24 & -- & -- & -- & -- \\
Shortcut-XL/2~\citep{frans2025one} & latent & 250 & 1 & No & 10.60 & -- & -- & -- & -- \\
MeanFlow~\citep{meanflow} & latent & 240 & 1 & No & 3.43 & -- & -- & -- & -- \\
iMF-XL~\citep{geng2026improved} & latent & 800 & 1 & No & 1.82 & 8.39 & 278.9 & 0.78 & 0.63 \\
Drift-L (latent)~\citep{deng2026generative} & latent & 1280 & 1 & latent-MAE & 1.53 & 10.92 & 257.2 & 0.79 & 0.63 \\
Drift-L (pixel)~\citep{deng2026generative} & pixel & 640 & 1 & pixel-MAE & 1.43 & 10.51 & 305.8 & 0.81 & 0.60 \\
pMF-L~\citep{lu2026one} & pixel & 320 & 1 & ConvNeXt-v2+VGG & 2.72 & 9.09 & 261.7 & 0.81 & 0.56 \\
pMF-H~\citep{lu2026one} & pixel & 360 & 1 & ConvNeXt-v2+VGG & 2.29 & 6.87 & 267.2 & 0.80 & 0.59 \\
\midrule
\multicolumn{10}{l}{\emph{Distilled or post-trained one-step models}} \\
JiT-L + FD-loss$^\dagger$~\citep{yang2026representation} & pixel & 600+100 & 1 & (SigLIP+Incep.+MAE) & 0.77 & 3.24 & 317.3 & 0.77 & 0.66 \\
iMF-XL + FD-loss$^\dagger$~\citep{yang2026representation} & latent & 800+100 & 1 & (SigLIP+Incep.+MAE) & 0.76 & 2.45 & 301.3 & 0.77 & 0.67 \\
pMF-L + FD-loss$^\dagger$~\citep{yang2026representation} & pixel & 320+100 & 1 & (SigLIP+Incep.+MAE) & 0.78 & 2.09 & 309.2 & 0.76 & 0.67 \\
\midrule
\multicolumn{10}{l}{\emph{RepFlow-JiT}} \\
RepFlow-JiT-XL, naive one-step & pixel & 400 & 1 & No & 268.44 & 217.75 & 2.42 & -- & -- \\
RepFlow-JiT-XL + MAE FD & pixel & 400+30 & 1 & MAE & 6.34 & 9.33 & 286.1 & 0.77 & 0.56 \\
\textbf{RepFlow-JiT-XL + RepFlow-student FD} & pixel & 400+30 & 1 & No & 3.14 & 8.06 & 312.8 & 0.82 & 0.56 \\
\bottomrule
\end{tabular}}
\vspace{-2mm}
\end{table}

Direct one-step sampling from the starting checkpoint gives FID 268.44 (Table~\ref{tab:onestep_system}). After post-training, replacing MAE with the RepFlow student lowers FID from 6.34 to 3.14, a 50.5\% reduction, and $\mathrm{FDr}^{6}$ from 9.33 to 8.06. IS and precision also increase, while recall remains 0.56. The student thus retains supervisory value after it is frozen and its role changes from instance-level feedback to a distribution-level objective.

The remaining systems in Table~\ref{tab:onestep_system} provide context across different architectures, training budgets, and objectives. The daggered FD-loss systems train in feature spaces that also appear in the reported FID/$\mathrm{FDr}^{6}$ evaluations. RepFlow-student FD uses a separate learned feature space and none of the six evaluation encoders in its objective.
Figure~\ref{fig:selected_onestep_samples} shows selected samples from the one-step RepFlow-JiT-XL model post-trained with RepFlow-student FD.

\begin{figure}[!ht]
\centering
\includegraphics[width=0.9\linewidth]{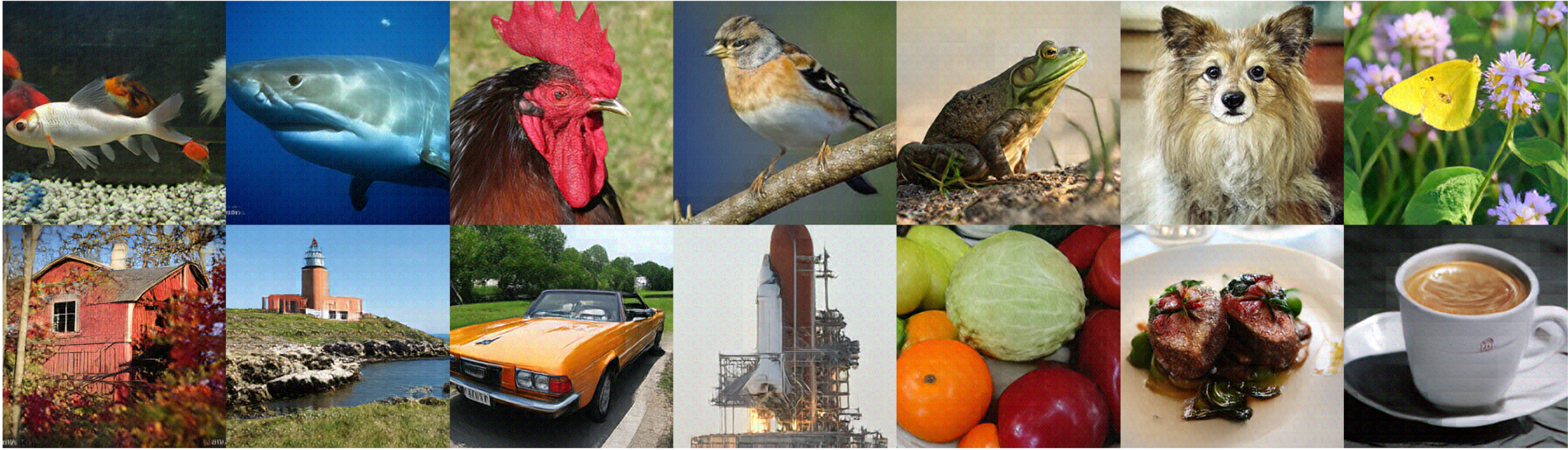}
\caption{Selected one-step samples from RepFlow-JiT-XL with RepFlow-student FD on ImageNet at $256\times256$ resolution.}
\label{fig:selected_onestep_samples}
\vspace{-5mm}
\end{figure}

\section{Conclusion}

RepFlow learns a generator's representation guidance from its own evolving computation. Predicting states across depth and noise level from masked clean images trains a separate, timestep-free encoder; feedback from this encoder then changes the states that supervise subsequent representation learning. The supervisory signal is therefore an outcome of the same training process it helps guide.

On ImageNet, this process improves multi-step generation and frozen linear probing with both latent-space SiT and pixel-space JiT, without an externally pretrained representation teacher. Ablations show that feedback benefits both models and favor generator-state prediction across multiple depths over the tested pixel-reconstruction and single-depth alternatives. Once frozen, the learned student also supports one-step JiT distribution matching, extending the value of the representation beyond the joint training objective that produced it.

\section{AI Use Statement.}
Generative AI tools were used during manuscript preparation to draft and edit
text for readability, refine the presentation of existing material, and assist
with formatting tables, citations, and document layout. They were not used for generating or cleaning
data, developing conceptual or theoretical models, translating research content, analyzing data, or interpreting results. The
authors reviewed and verified all AI-assisted work and take full responsibility
for the final text, claims, results, and artifacts in this paper.
\bibliography{iclr2027_conference}
\bibliographystyle{iclr2027_conference}
\appendix
\section{Related Work}

\subsection{Predictive Representation Learning and Generative Features}

Self-supervised representation learning includes contrastive objectives such as SimCLR and MoCo~\citep{simclr,mocov3}, non-contrastive methods such as BYOL, SimSiam, and DINO~\citep{byol,simsiam,dino}, and masked prediction~\citep{mae}. I-JEPA~\citep{ijepa}, for example, predicts latent targets rather than pixels. RepFlow adopts this predictive perspective but takes its targets from the generator's internal computation rather than from another clean-view encoder.

Pretrained diffusion and flow models also provide useful features~\citep{zhang2022unsupervised,chen2025deconstructing,zeng2025flow} for correspondence, segmentation, and classification~\citep{baranchuk2021label,tang2023emergent,mukhopadhyay2024text}, although feature quality depends strongly on layer and timestep. DIFT selects denoising states for a downstream task~\citep{tang2023emergent}, whereas CleanDIFT distills noisy generator features into a clean-image encoder~\citep{stracke2025cleandift}. RepFlow differs by training its student jointly from multiple generator depths and returning the resulting feature to the generator.

\subsection{Representation-guided Generative Modeling}

Representation-level objectives can improve generation. REPA aligns generator states with a frozen visual encoder~\citep{repa,irepa}, while Representation Autoencoders replace conventional VAE encoders with pretrained representation encoders~\citep{vavae,rae}. RepFlow considers a different source of supervision: the guiding encoder is learned concurrently from the generator's own states. This removes external representation pretraining from the method, but not the computational cost of training the additional student.

Several methods obtain representation guidance from the generator itself. SRA aligns an early, noisy state with a later EMA state at a cleaner timestep~\citep{sra}; Dispersive Loss regularizes internal feature geometry~\citep{disperse}; Self-Flow creates token-wise noise asymmetry within a flow model~\citep{selfflow}; and D-JEPA adds a predictive objective to generative training~\citep{djepa}. These methods avoid a frozen external encoder, but keep representation learning inside the generative backbone. Their behavior can also depend on the generative parameterization: alignment recipes developed for latent diffusion do not transfer unchanged to pixel-space JiT~\citep{pixelrepa}.

Architecturally, RepFlow keeps the generator and representation encoder separate. G2R trains the student from multi-depth generator states, and R2G returns the student feature to the generator. By contrast, D-JEPA and Self-Flow place predictive or asymmetric-noise objectives inside a single generative backbone, while CleanDIFT adapts a pretrained generator for feature extraction. We study the same explicit student--generator interface in both latent-space SiT and pixel-space JiT.

\subsection{Flow Models and One-Step Generation}

One-step generation replaces numerical integration with a single noise-to-data evaluation. Existing approaches use progressive or consistency distillation, distribution matching, and adversarial or score-based objectives~\citep{salimans2022progressive,song2023consistency,icm,yin2024one,yin2024improved,yang2026representation,sauer2024adversarial,sauer2024fast,ren2024hyper,xu2024ufogen}. Others learn finite-time maps directly: Shortcut Models condition on step size, flow-map methods learn finite transitions, and MeanFlow predicts average velocity~\citep{frans2025one,boffi2024flow,meanflow}. Improved MeanFlow, Pixel MeanFlow, and Drifting Models scale these ideas to larger latent- and pixel-space models~\citep{geng2026improved,lu2026one,deng2026generative}. RepFlow leaves the joint-training objective multi-step and uses the learned student only as the metric for a short JiT post-training stage.

\section{Objective and Optimization Analysis}
\label{sec:theory}

We analyze two aspects of RepFlow: the information rewarded by G2R and the
way stop-gradient couples G2R and R2G across optimization steps. The results
characterize the training mechanism; they do not imply semantic utility or
improved generation, which remain empirical questions.

Let $X$ be a clean image, $\widetilde X=\mathcal{M}(X;\zeta)$ its masked view,
and $T$ the generative timestep. The multi-level G2R target is
\begin{equation}
Y_T=\operatorname{Concat}\left[Y_T^{(1)},\ldots,Y_T^{(K)}\right],
\qquad
Y_T^{(k)}=h_{l_k}^G(X_T,T),
\end{equation}
where the randomness in $X_T$ includes the sampled reference noise. The
student code is $U=U_\phi(\widetilde X)$ and the predictor has the form
$f(U,T)$. All expectations below include data, mask, timestep, and reference
noise randomness.

\subsection{G2R as Explained Teacher Variation}

For a fixed representation $U$, define the optimal squared prediction risk
\begin{equation}
\mathcal{R}_2(U)
:=
\inf_f\;\mathbb{E}\left[\|Y_T-f(U,T)\|_2^2\right].
\end{equation}
The following decomposition identifies exactly what reducing this risk
requires from the representation.

\begin{proposition}[Conditional-risk decomposition]
\label{prop:g2r_risk}
For every square-integrable predictor $f(U,T)$,
\begin{align}
\mathbb{E}\|Y_T-f(U,T)\|_2^2
&=
\mathbb{E}\,\operatorname{Tr}\!\left(\operatorname{Var}[Y_T\mid U,T]\right)
\nonumber\\
&\quad+
\mathbb{E}\left\|f(U,T)-\mathbb{E}[Y_T\mid U,T]\right\|_2^2.
\end{align}
Consequently,
\begin{equation}
f_U^\star(U,T)=\mathbb{E}[Y_T\mid U,T],
\qquad
\mathcal{R}_2(U)
=
\mathbb{E}\,\operatorname{Tr}\!\left(\operatorname{Var}[Y_T\mid U,T]\right).
\end{equation}
If $\mathcal{R}_2(T)$ denotes the risk of a predictor that observes the
timestep but not the image representation, then
\begin{equation}
\mathcal{R}_2(T)-\mathcal{R}_2(U)
=
\mathbb{E}\left\|
\mathbb{E}[Y_T\mid U,T]-\mathbb{E}[Y_T\mid T]
\right\|_2^2
\ge 0.
\label{eq:explained-teacher-variation}
\end{equation}
\end{proposition}

\begin{proof}
Set $m(U,T)=\mathbb{E}[Y_T\mid U,T]$ and expand
$Y_T-f=(Y_T-m)+(m-f)$. The cross term has zero expectation after conditioning
on $(U,T)$, which gives the first identity and its minimizer. Applying the same
orthogonal decomposition to $m(U,T)-\mathbb{E}[Y_T\mid T]$ yields
Equation~\ref{eq:explained-teacher-variation}.
\end{proof}

Equation~\ref{eq:explained-teacher-variation} has a direct interpretation:
G2R rewards the student for explaining teacher variation beyond timestep alone. Because $U$ is a function of $\widetilde X$,
the same orthogonal decomposition yields
\begin{equation}
\mathcal{R}_2(U)
\!=\!
\mathcal{R}_2(\widetilde X)
+
\mathbb{E}\left\|
\mathbb{E}[Y_T\mid \widetilde X,T]
-
\mathbb{E}[Y_T\mid U,T]
\right\|_2^2
\!\ge\!
\mathcal{R}_2(\widetilde X),
\label{eq:irreducible-g2r}
\end{equation}
where $\mathcal{R}_2(\widetilde X)$ is the Bayes risk when the predictor sees
the entire masked view. Equality holds if and only if
\begin{equation}
\mathbb{E}[Y_T\mid U,T]
=
\mathbb{E}[Y_T\mid \widetilde X,T]
\quad\text{almost surely}.
\label{eq:predictive-sufficiency}
\end{equation}
Hence any student code that attains the masked-view Bayes risk is a
mean-sufficient statistic for the cross-timestep prediction profile
$\{\mathbb{E}[Y_T\mid\widetilde X,T=t]\}_{t}$, it need not retain aspects of
the masked image that never change this profile. Conversely, target components
that depend on reference noise or unobserved image content contribute to the
irreducible term $\mathcal{R}_2(\widetilde X)$ and cannot be recovered by any
student representation of the masked view.

Experiments use Smooth\,L1 rather than squared loss. Its population minimizer
is a conditional Huber location, so the variance identities apply only to the
squared surrogate in Section~\ref{sec:method}. Under either loss, the predictor
acts only on $(U,T)$. The objective alone does not guarantee semantic features,
masking and a shared timestep-free student supply the inductive bias. We use
linear probing to evaluate whether the resulting representation separates
ImageNet classes.

No independence assumption between target layers is needed: both squared and
Smooth\,L1 losses decompose over target coordinates. The predictor may
specialize its outputs across layers and timesteps, but every target is decoded
from the same code $U$. This construction encourages a shared explanatory code
without asserting that every selected layer carries complementary information.
At inference, RepFlow exposes $\bar U_\phi(X)$ without a layer or timestep index,
the removal of stage selection is an architectural property, not by itself a
claim of semantic superiority.

\subsection{Gradient Routing and Reciprocal Dynamics}

Stop-gradient makes the instantaneous parameter routing exact:
\begin{align}
\nabla_\theta\mathcal{L}_{\mathrm{joint}}
&=
\nabla_\theta\mathcal{L}_{\mathrm{gen}}
+\lambda_2\nabla_\theta\mathcal{L}_{\mathrm{R2G}},
&
\nabla_\phi\mathcal{L}_{\mathrm{joint}}
&=\lambda_1\nabla_\phi\mathcal{L}_{\mathrm{G2R}},
\nonumber\\
\nabla_\psi\mathcal{L}_{\mathrm{joint}}
&=\lambda_1\nabla_\psi\mathcal{L}_{\mathrm{G2R}},
&
\nabla_\omega\mathcal{L}_{\mathrm{joint}}
&=\lambda_2\nabla_\omega\mathcal{L}_{\mathrm{R2G}}.
\label{eq:gradient-routing}
\end{align}
G2R has no direct generator gradient, and the EMA full-view branch has no
direct student gradient. Reciprocity arises across optimization steps.
Suppressing optimizer state, the updates have the schematic form
\begin{align}
(\phi_{k+1},\psi_{k+1})
&\leftarrow
(\phi_k,\psi_k)
-\eta_s\nabla_{\phi,\psi}
\mathcal{L}_{\mathrm{G2R}}(\phi_k,\psi_k;\theta_k),
\label{eq:student-update}\\
(\theta_{k+1},\omega_{k+1})
&\leftarrow
(\theta_k,\omega_k)
-\eta_g\nabla_{\theta,\omega}
\left[
\mathcal{L}_{\mathrm{gen}}(\theta_k)
+\lambda_2\mathcal{L}_{\mathrm{R2G}}(
\theta_k,\omega_k;\bar\phi_k)
\right],
\label{eq:generator-update}\\
\bar\phi_{k+1}
&\leftarrow
m\bar\phi_k+(1-m)\phi_{k+1}.
\label{eq:ema-update}
\end{align}
Semicolons mark stopped arguments: $\theta_k$ determines the current G2R
targets, and $\bar\phi_k$ determines the current R2G targets, but neither
receives a gradient through that occurrence. The loop nevertheless closes
across iterations. R2G changes $\theta_{k+1}$ and therefore the generator
states used by later G2R updates; G2R changes $\phi_{k+1}$ and therefore the EMA
target used by later generator updates. R2G can thus improve the eventual
student without sending it a direct gradient, consistent with
Table~\ref{tab:reciprocal_ablations}.

Any stationary fixed point of this idealized update, with the EMA tracking the
online student, satisfies
\begin{equation}
\nabla_{\phi,\psi}\mathcal{L}_{\mathrm{G2R}}=0,
\qquad
\nabla_{\theta,\omega}
\left(\mathcal{L}_{\mathrm{gen}}
+\lambda_2\mathcal{L}_{\mathrm{R2G}}\right)=0,
\qquad
\bar\phi=\phi.
\end{equation}
At such a fixed point, the student summarizes predictable variation in the
current generator, while the generator is stationary under both its native
objective and alignment to the current student.

The risk decomposition identifies the prediction profile preserved by a
Bayes-optimal G2R code, and the gradient equations show how two stopped paths
still create feedback over time.

% -----------------------------------------------------------------------------
% Inlined implementation and evaluation appendix
% (formerly appendix_experimental_details.tex)
% -----------------------------------------------------------------------------
\section{Implementation and Evaluation Details}
\label{app:exp}

Unless noted otherwise, all experiments use ImageNet-1K at
$256{\times}256$ resolution, a global batch size of $1024$, bfloat16 mixed
precision, and eight NVIDIA H200 GPUs. Transformer layers are indexed from
zero, with an index referring to the output of that block. The main text
parameterizes the path from data ($t{=}0$) to noise ($t{=}1$); the implementation
uses $\tau=1-t$, and all timesteps below are reported in $\tau$.
Following \citet{sun2025noiseconditioningnecessarydenoising}, we omit explicit
noise-level (timestep) conditioning from the denoising network in all
flow-matching experiments.

\subsection{Data and Generative Backbones}
\label{app:backbones}

Training uses the ImageNet-1K split ($1{,}281{,}167$ images, $1{,}000$
classes)~\citep{imagenet}. We apply the ADM center crop at
$256{\times}256$, random horizontal flipping, and rescaling to $[-1,1]$.

\paragraph{SiT.}
SiT-XL/2~\citep{sit} operates on $32{\times}32{\times}4$ latents from the
frozen \texttt{sd-vae-ft-ema} VAE with scale $0.18215$. Its 28 blocks use
width $1152$, 16 heads, MLP ratio 4, and patch size 2. We use the official
linear interpolant and velocity objective with $\tau\sim\mathrm{Unif}[0,1]$.
The student receives the corresponding RGB image; both branches contain
$16{\times}16=256$ spatial tokens.

\paragraph{JiT.}
JiT-XL/16~\citep{jit} acts directly on RGB patches. Its 28 blocks use width
$1152$, 16 heads, a 128-dimensional bottleneck patch embedding, 2D RoPE, QK
normalization, and SwiGLU MLPs. In implementation coordinates,
\begin{equation}
z_\tau=\tau x_0+(1-\tau)x_1,
\qquad
\tau=\sigma(\varepsilon),\quad
\varepsilon\sim\mathcal{N}(-0.8,0.8^2).
\end{equation}
The network predicts $x_0$ and minimizes unweighted velocity MSE with
$v=(x_0-z_\tau)/(1-\tau)$, clamping the denominator at $0.05$. Equivalently,
the $x_0$ regression loss is weighted by $(1-\tau)^{-2}$. The attached RepFlow
modules use the same optimizer schedule summarized in Table~\ref{tab:app-optim}.
Conditional training uses label dropout $0.1$; unconditional training uses the
null label. All reported samples are generated without classifier-free guidance.

\begin{table}[!ht]
\centering
\small
\caption{Optimization and sampling settings. The EMA decay applies to the
generator and, in RepFlow, to the full-view student target.}
\label{tab:app-optim}
\setlength{\tabcolsep}{4pt}
\begin{tabular}{lcc}
\toprule
 & SiT-XL/2 & JiT-XL/16 \\
\midrule
AdamW $\beta$ & $(0.9,0.999)$ & $(0.9,0.95)$ \\
Learning rate & $1{\times}10^{-4}$ & $2{\times}10^{-4}$ \\
Schedule & constant & cosine, $5$ warmup epochs \\
Minimum learning rate & -- & $1{\times}10^{-6}$ \\
Weight decay & $0$ & $0$ \\
EMA decay & $0.9999$ & $0.9999$ \\
Multi-step sampler & dopri5, $250$ NFE & Heun, $50$ NFE \\
\bottomrule
\end{tabular}
\end{table}

\subsection{RepFlow Instantiation}
\label{app:repflow-impl}

\paragraph{Student views.}
The student ViT receives no timestep, class, or adaLN conditioning. ViT-B/16
has 12 blocks, width 768, and 12 heads (85.6M parameters); ViT-L/16 has 24
blocks, width 1024, and 16 heads (303.0M). SiT students use fixed sine--cosine
positional embeddings, whereas JiT students use 1D RoPE. G2R masks 75\% of the
256 RGB tokens and encodes only the 64 visible tokens, without mask tokens. R2G
and evaluation process the complete image. Across all configurations,
$U_\phi$ is strictly an RGB pixel-space encoder: it never consumes VAE latents
or any other generator-space inputs, including when the generator is SiT.

\paragraph{Generation-to-representation path.}
G2R concatenates generator states from
\begin{equation}
\mathcal{D}=\{5,12,19,26\},
\end{equation}
forming a 4608-dimensional target at each spatial position. The targets are
channel-wise layer-normalized and detached. The predictor restores the
$16{\times}16$ grid, fills masked positions with sinusoidal timestep embeddings,
and applies ten width-384 blocks with 12 heads. Smooth\,L1 loss ($\beta=1$) is
computed only at masked positions.

In every configuration, the predictor reads the final student block and uses
width 384 with 12 heads.

\paragraph{Representation-to-generation path.}
R2G aligns the generator state after block $7$ with the full-view student
representation. A two-layer head projects the generator state to the student
dimension; for JiT, its hidden width is $1024$. Both sides are
$\ell_2$-normalized per token, and the loss is cosine distance. The stopped
full-view target comes from an EMA student with decay $0.9999$.

The implemented joint objective is therefore
\begin{equation}
\mathcal{L}
=
\mathcal{L}_{\mathrm{gen}}
+\lambda_1\mathcal{L}^{\mathrm{SmoothL1}}_{\mathrm{G2R}}
+\lambda_2\mathcal{L}^{\mathrm{cos}}_{\mathrm{R2G}},
\qquad
\lambda_1=1,\quad \lambda_2=0.1.
\end{equation}
Section~\ref{sec:method} uses the same Smooth\,L1 G2R loss and cosine-distance
R2G loss as the implemented objective above. Each joint update evaluates the
generator once, the online student on a
masked view, the predictor once, and the EMA student on the complete image.

\subsection{Baselines and Evaluation Protocols}
\label{app:evaluation}

\paragraph{Generation controls.}
For unconditional experiments, the generator control retains G2R but sets
$\lambda_2=0$. Because the G2R target is detached, the generator update is
identical to native training, while the run still produces a student for
representation evaluation. The class-conditional native baseline omits the
student. In both comparisons, R2G is the only additional generator gradient.

REPA~\citep{repa} and SRA~\citep{sra} match the native backbone, data,
batch size, EMA, and sampler. REPA projects block 7 through a three-layer MLP
and aligns it with frozen DINOv2-ViT-L/14 patch tokens~\citep{dinov2} using a
cosine-loss coefficient of 0.5. SRA aligns block 7 with an EMA block-19 state at
$\tau'=(\tau+\delta)\wedge1$, $\delta\sim\mathrm{Unif}[0,0.2]$, using the same coefficient.

\paragraph{Generation metrics.}
FID-50K, spatial FID, Inception Score, precision, and recall are computed from
$50{,}000$ EMA samples with the ADM evaluator and the fixed ImageNet reference
statistics. SiT uses the adaptive ODE sampler and JiT uses Heun as specified
in Table~\ref{tab:app-optim}. One-step JiT evaluation uses a single
terminal-noise forward pass ($\mathrm{NFE}=1$). All evaluations use
$\mathrm{CFG}=1$.

\paragraph{Frozen linear probing.}
We freeze the EMA student and average its $256$ final-layer tokens. A linear
classifier is then trained for $30$ epochs with Adam, learning rate $10^{-3}$,
cosine decay, batch size $1024$, and cross-entropy. Training uses the same
center crop and horizontal flip as representation learning; validation uses
only the center crop.

The raw generator baseline searches over layer and timestep on the
unconditional G2R-only checkpoint. The selected features are layer $21$ at
$\tau=0.7$ for SiT and layer $21$ at $\tau=0.5$ for JiT. This selection is
used only for the raw-feature baseline; RepFlow always evaluates a single
full-view student feature.

\paragraph{Ablation controls.}
The pixel-target control retains the ViT-L student, $75\%$ masking ratio,
R2G branch, data, and optimization budget, but replaces the detached
multi-level generator target in G2R with masked RGB-patch reconstruction.
The target-depth control retains the full reciprocal objective and predicts
only the strongest evaluated single generator layer: block $19$ for SiT and
block $20$ for JiT. The multi-level variant uses blocks $\{5,12,19,26\}$.

\subsection{One-Step Post-Training}
\label{app:onestep}

The class-conditional one-step experiment begins from the 400-epoch
RepFlow-JiT checkpoint, with the ViT-L student frozen. Real statistics are
computed once from layer-normalized, mean-pooled ImageNet features. Generated
first and second moments are tracked with EMA decay 0.999 over an effective
population of 50,000 samples. Only the generator is updated by the mean and
squared-Bures terms from Section~\ref{sec:method}; evaluation uses
$\mathrm{CFG}=1$.

The matched control keeps the initialization, generated samples, optimizer,
schedule, and duration fixed, replacing only the frozen RepFlow ViT-L/16
feature with an ImageNet-pretrained MAE ViT-L/16 feature. Real statistics use
the same protocol.

Post-training uses AdamW with learning rate $10^{-5}$, cosine decay to zero,
$6{,}250$ warmup steps, and global batch size $1024$. The reported model is
trained for $30$ additional epochs ($37{,}500$ steps). $\mathrm{FDr}^{6}$ averages
normalized Fr\'echet distances measured in Inception, ConvNeXt, DINOv2, MAE,
SigLIP, and CLIP feature spaces. FD-SIM directly optimizes SigLIP, Inception,
and MAE distributions, overlapping FID and $\mathrm{FDr}^{6}$ evaluation. RepFlow-student
FD uses a learned feature space distinct from all six evaluation encoders.

\subsection{Training Cost and Statistical Scope}
\label{app:cost}

RepFlow trades additional training cost for a jointly learned representation.
At 600 epochs, SiT takes about 100, 128, and 141 hours for Native, RepFlow-L,
and conditional RepFlow-L, respectively; JiT takes 66 and 99--105 hours for
Native and RepFlow-L. Throughput excludes sampling, and all runs use seed 0.

\begin{center}
\captionof{table}{Training cost and wall-clock matching at global batch size $1024$ on eight H200 GPUs. \textbf{Top:} measured system cost; SiT memory was not logged. \textbf{Bottom:} the RepFlow checkpoint with the same measured wall-clock as a 600-epoch Native run, computed as $600 \times s_{\mathrm{RF}} / s_{\mathrm{Native}}$.}
\label{tab:app-cost}
\small
\setlength{\tabcolsep}{5.0pt}
\renewcommand{\arraystretch}{0.92}
\begin{tabular}{lccccc}
\toprule
& \multicolumn{2}{c}{SiT} & \multicolumn{3}{c}{JiT} \\
\cmidrule(lr){2-3}\cmidrule(lr){4-6}
Method & Params & Steps/s & Params & Steps/s & Memory \\
\midrule
Native & $673$M & $2.09$ & $675$M & $3.15$ & $61.4$\,GB \\
REPA & $682$M & $1.82$ & ${\sim}680$M & $2.59$ & $64.5$\,GB \\
SRA & $678$M & $1.78$ & ${\sim}680$M & $2.45$ & $62.8$\,GB \\
RepFlow-B & $781$M & $1.82$ & $785$M & $2.70$ & $69.3$\,GB \\
RepFlow-L & $999$M & $1.63$ & $1002$M & $2.12$ & $80.1$\,GB \\
RepFlow-L (cond.) & $999$M & $1.48$ & $1002$M & $1.98$ & $83.9$\,GB \\
\bottomrule
\end{tabular}

\vspace{5pt}
\setlength{\tabcolsep}{5.0pt}
\begin{tabular}{lcccc}
\toprule
Setting & Native rate & RepFlow rate & Native budget & Equal-time RepFlow \\
\midrule
Unconditional SiT & $2.09$ steps/s & $1.63$ steps/s & 600 epochs & 468 epochs \\
Unconditional JiT & $3.15$ steps/s & $2.12$ steps/s & 600 epochs & 404 epochs \\
Class-conditional SiT & $2.09$ steps/s & $1.48$ steps/s & 600 epochs & 425 epochs \\
Class-conditional JiT & $3.15$ steps/s & $1.98$ steps/s & 600 epochs & 377 epochs \\
\bottomrule
\end{tabular}
\end{center}

At all four matched checkpoints, the RepFlow curves in
Figures~\ref{fig:uncond} and~\ref{fig:fid_comparison} remain below the
600-epoch Native endpoint. The generative gain therefore persists when RepFlow
is shortened to match Native wall-clock compute, rather than only when the two
methods receive equal data exposure.

\section{Uncurated One-Step Samples}
\label{app:uncurated-samples}

The class grids below show random one-step ImageNet samples
($\mathrm{NFE}=1$). All generated samples are displayed as produced, without
manual selection or post hoc filtering.

\begin{figure}[p]
\centering
\includegraphics[width=0.48\linewidth]{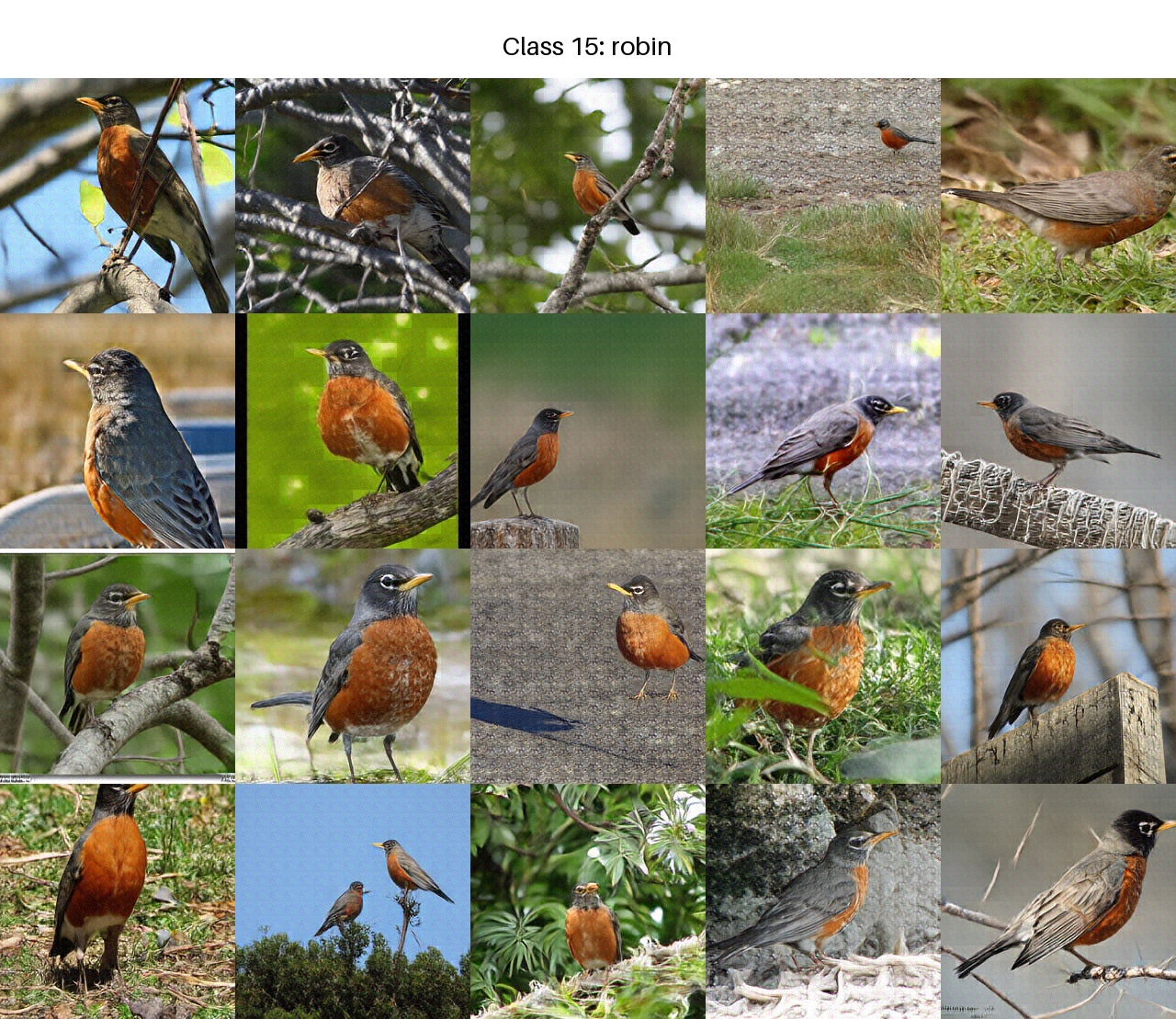}\hfill
\includegraphics[width=0.48\linewidth]{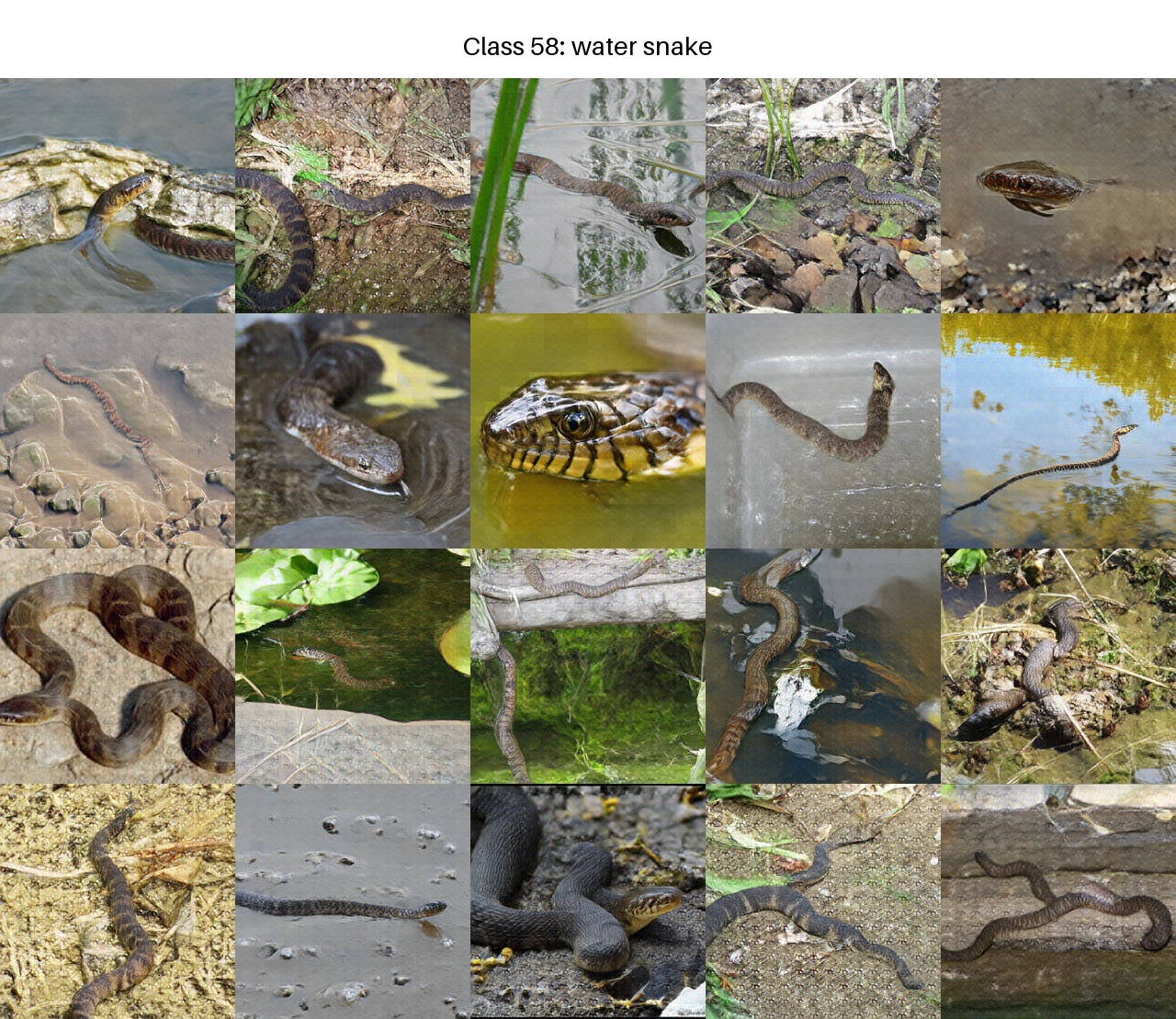}\\
\includegraphics[width=0.48\linewidth]{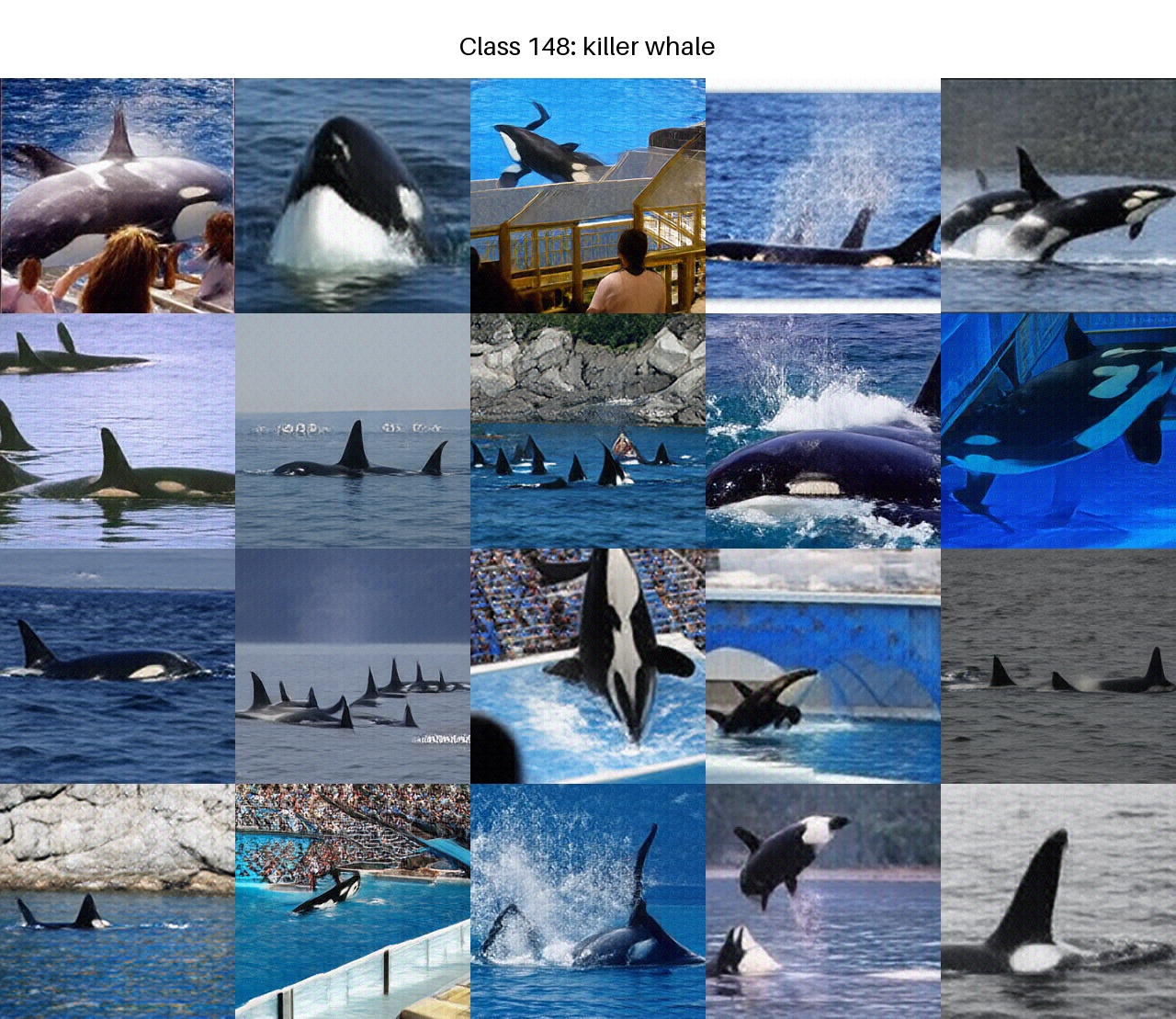}\hfill
\includegraphics[width=0.48\linewidth]{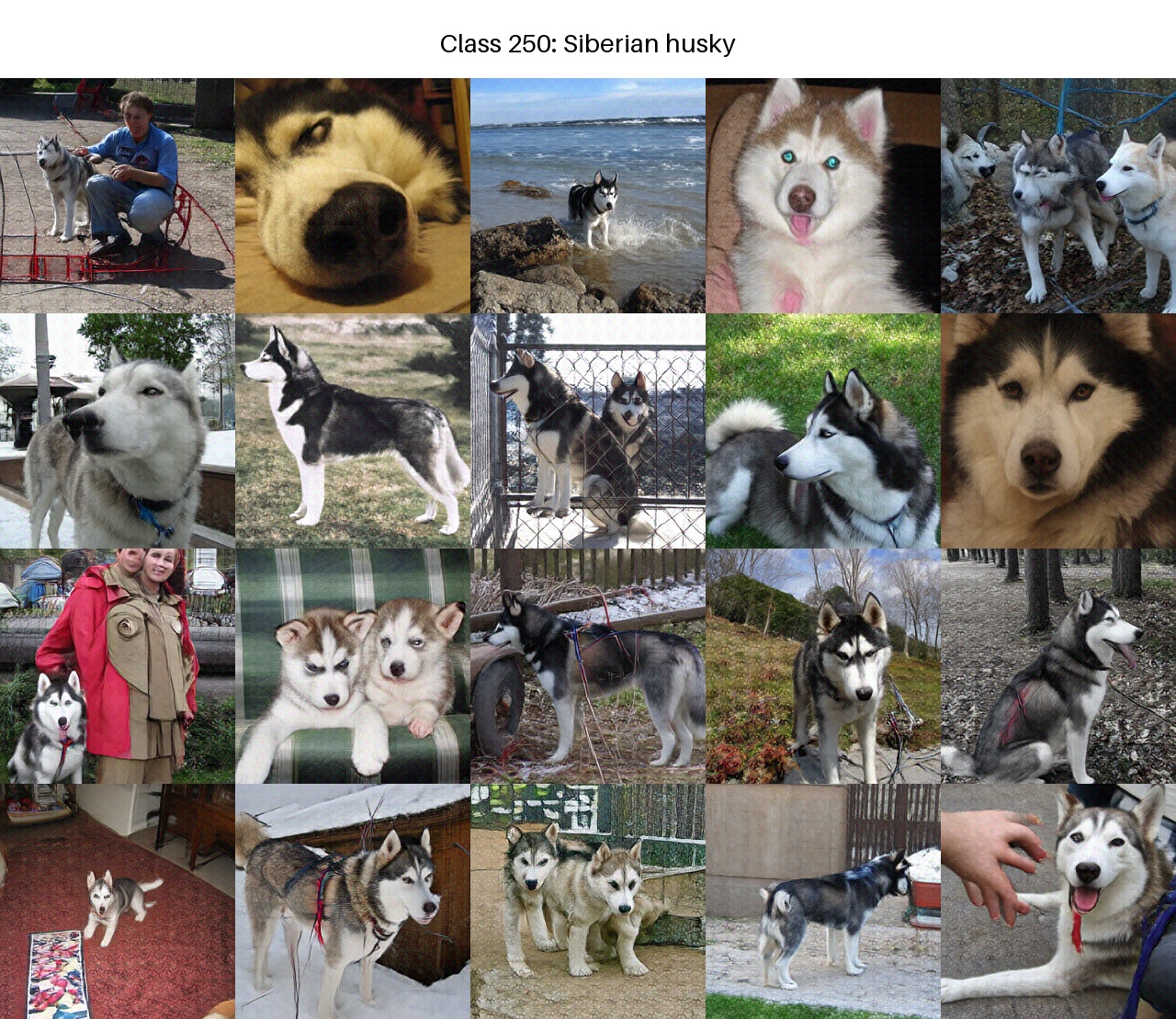}\\
\includegraphics[width=0.48\linewidth]{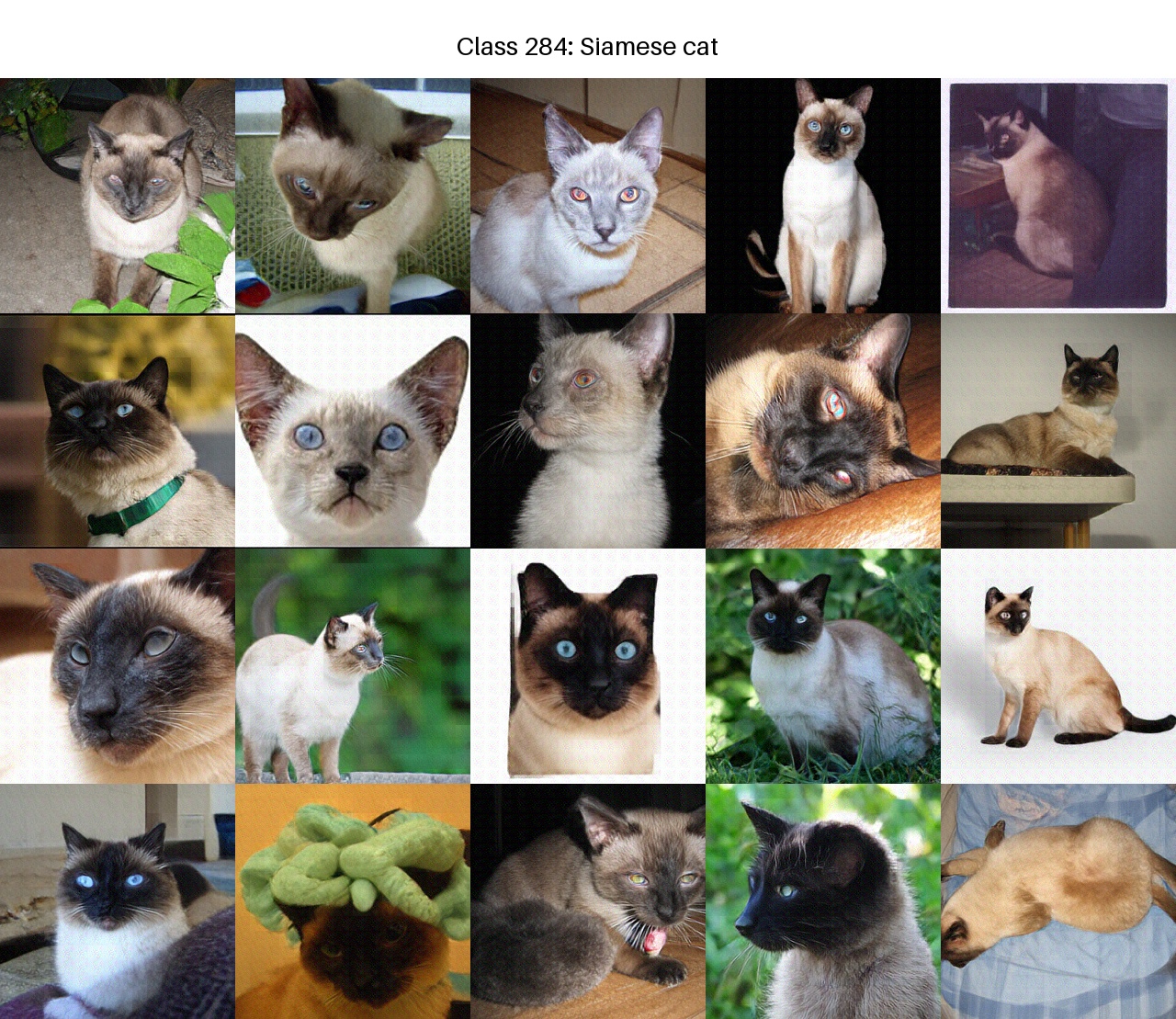}\hfill
\includegraphics[width=0.48\linewidth]{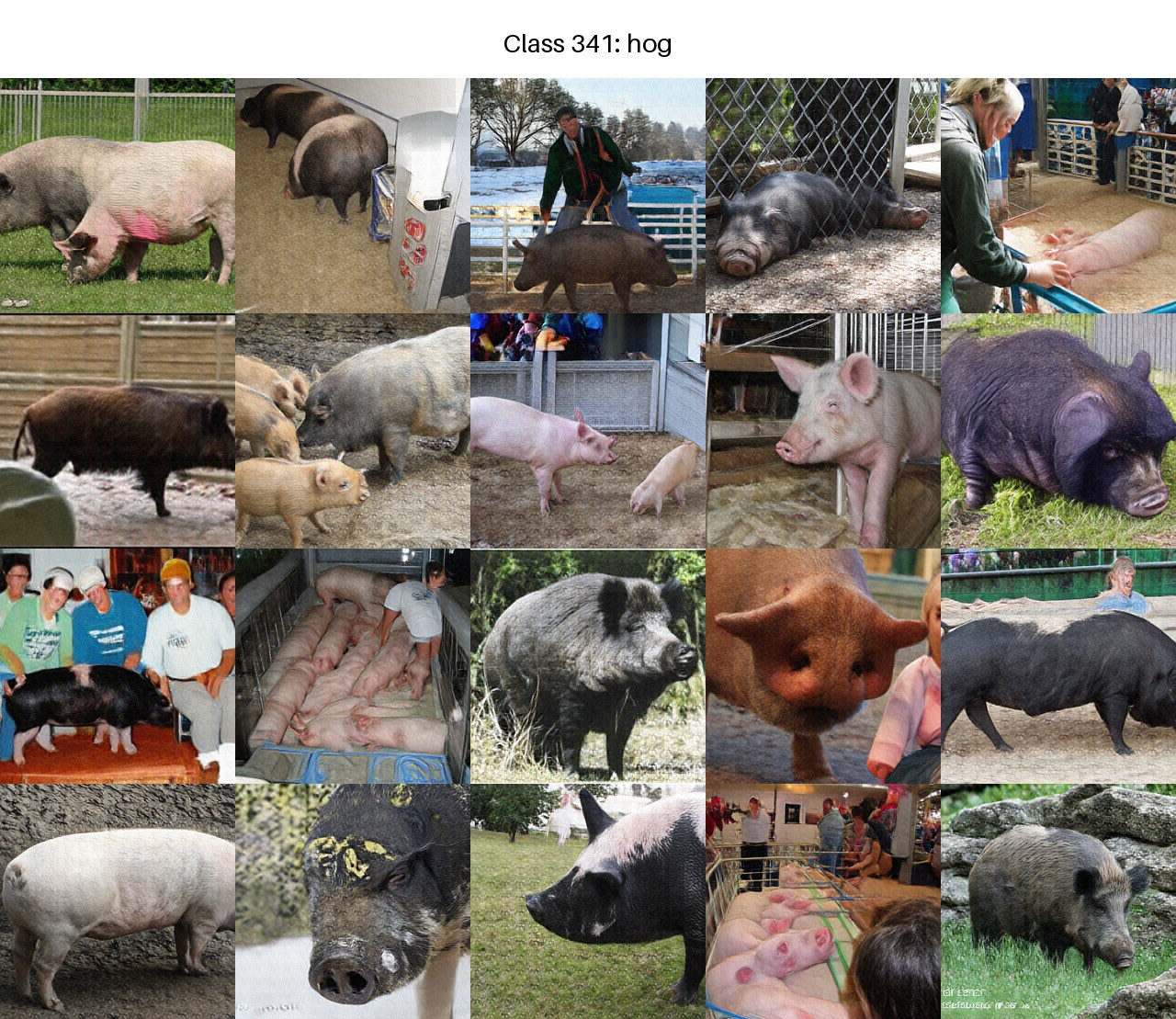}
\caption{Random, uncurated one-step ImageNet samples, classes 15--341.}
\label{fig:uncurated-onestep-1}
\end{figure}

\begin{figure}[p]
\centering
\includegraphics[width=0.48\linewidth]{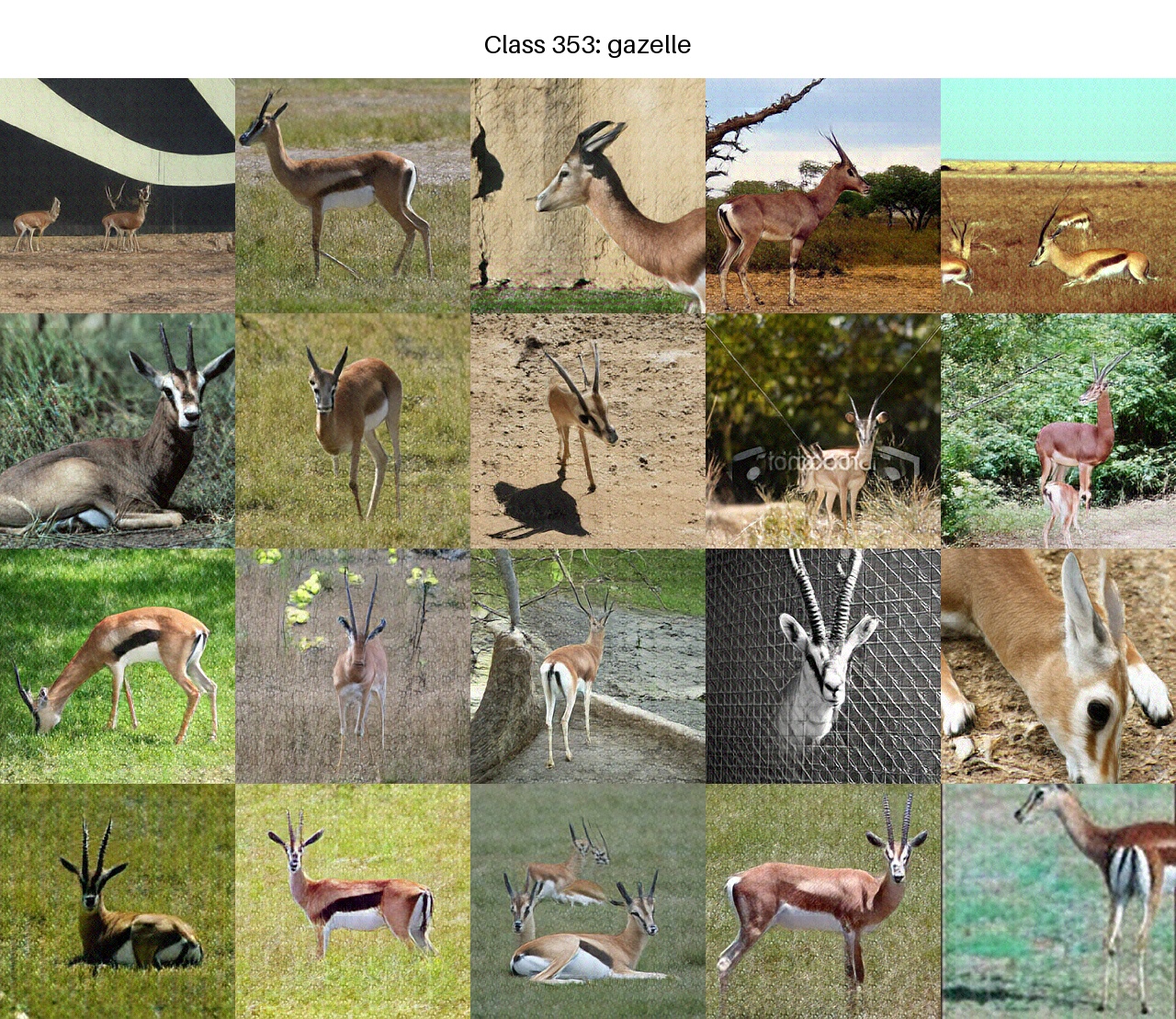}\hfill
\includegraphics[width=0.48\linewidth]{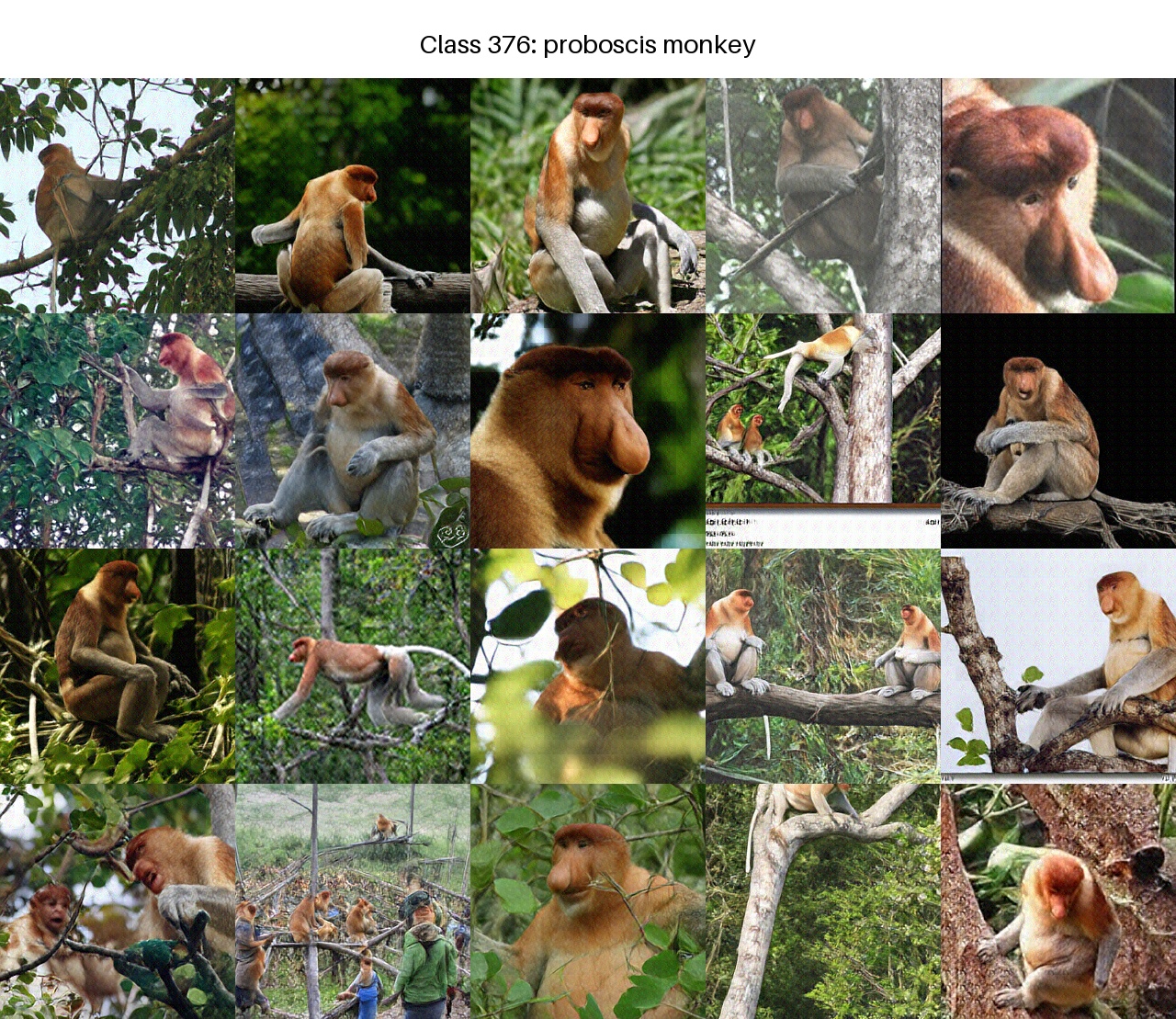}\\
\includegraphics[width=0.48\linewidth]{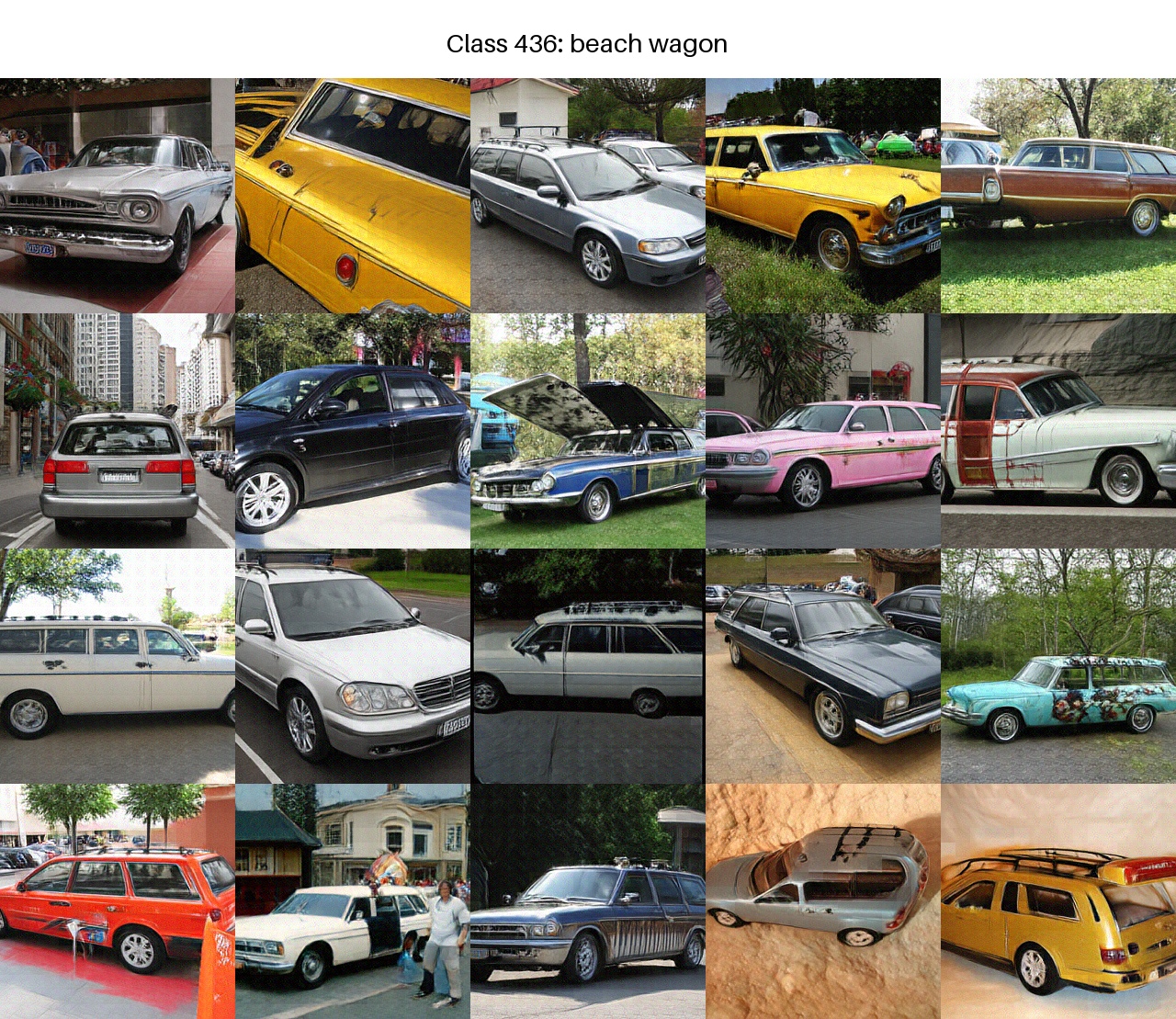}\hfill
\includegraphics[width=0.48\linewidth]{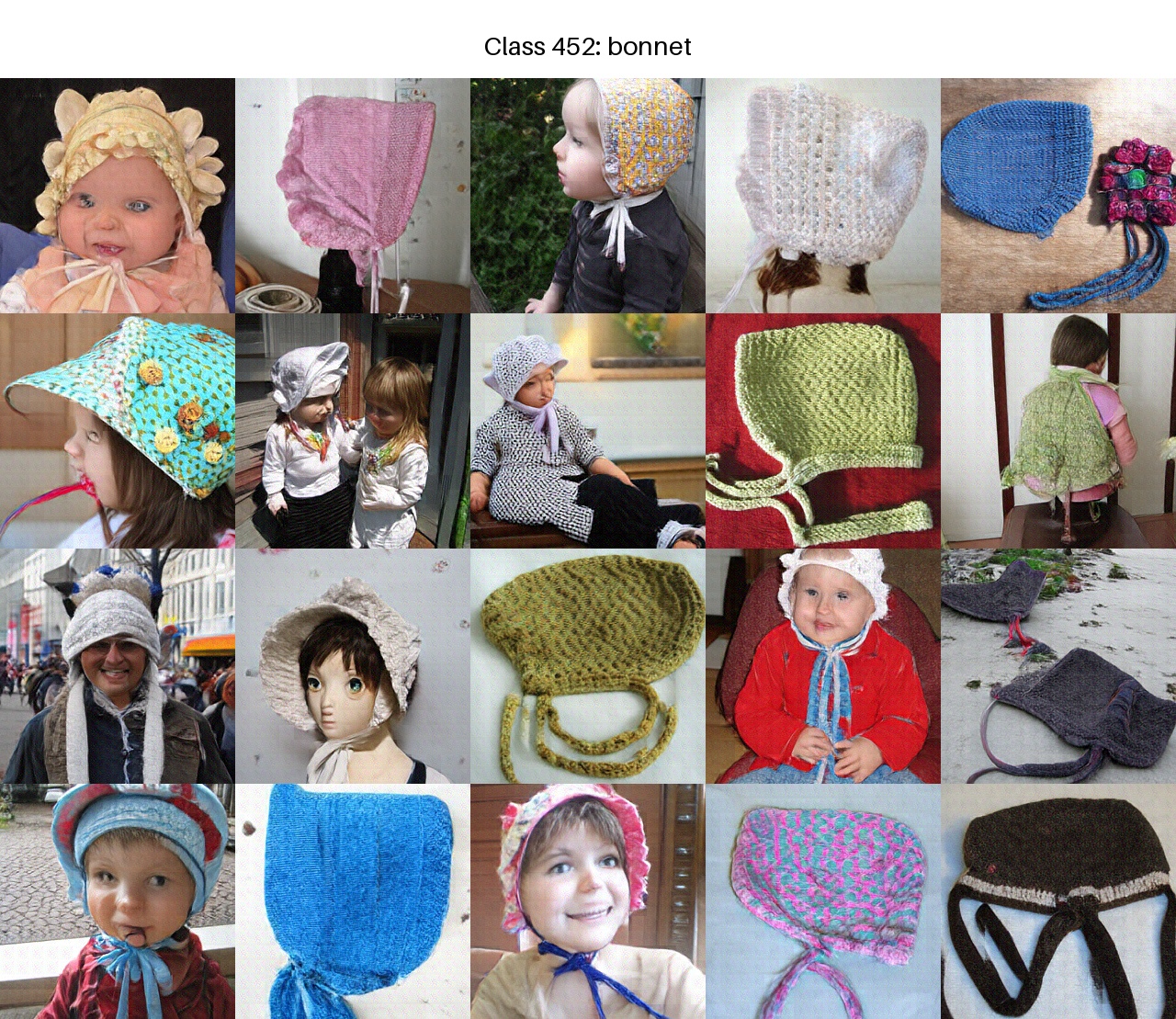}\\
\includegraphics[width=0.48\linewidth]{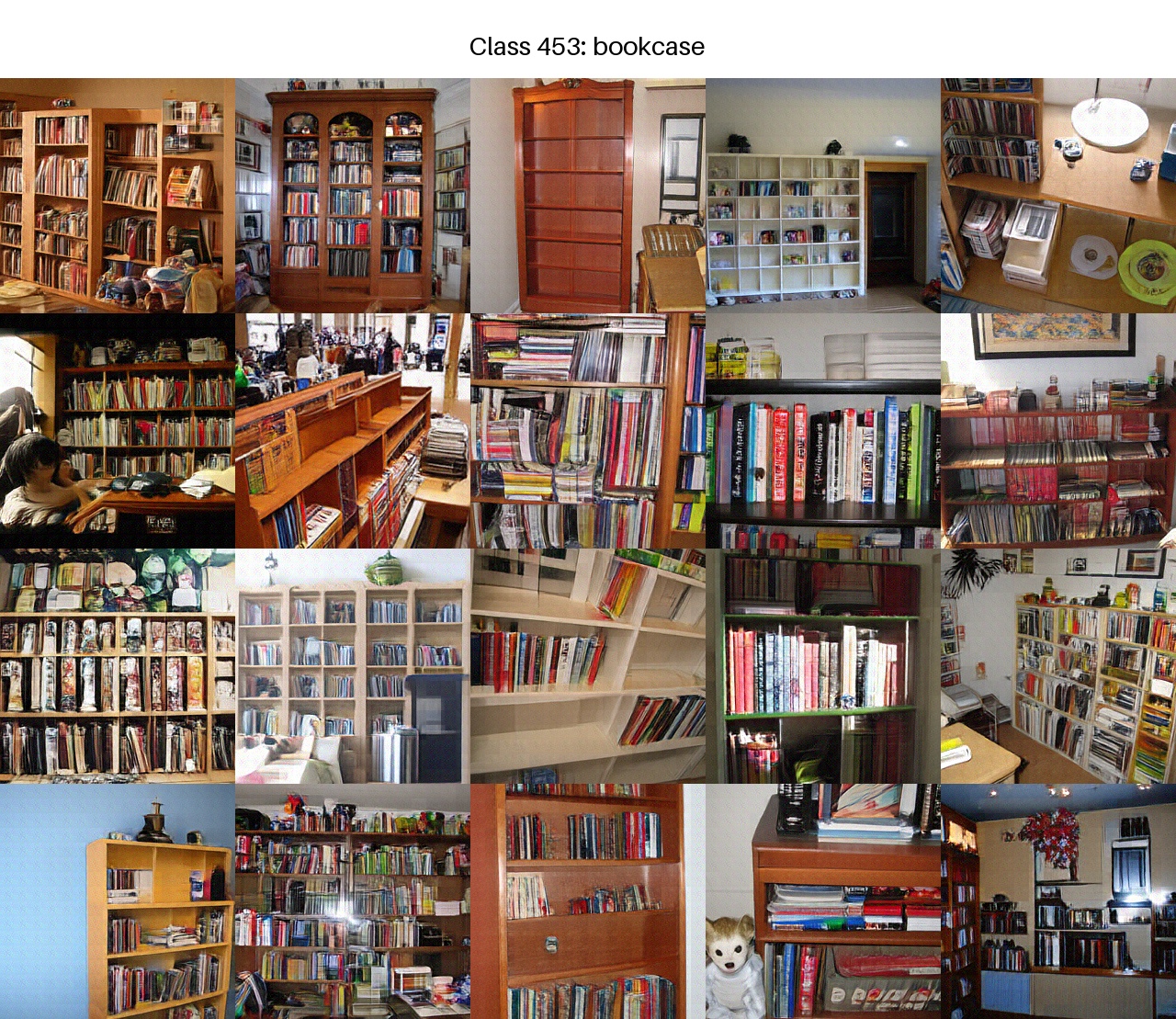}\hfill
\includegraphics[width=0.48\linewidth]{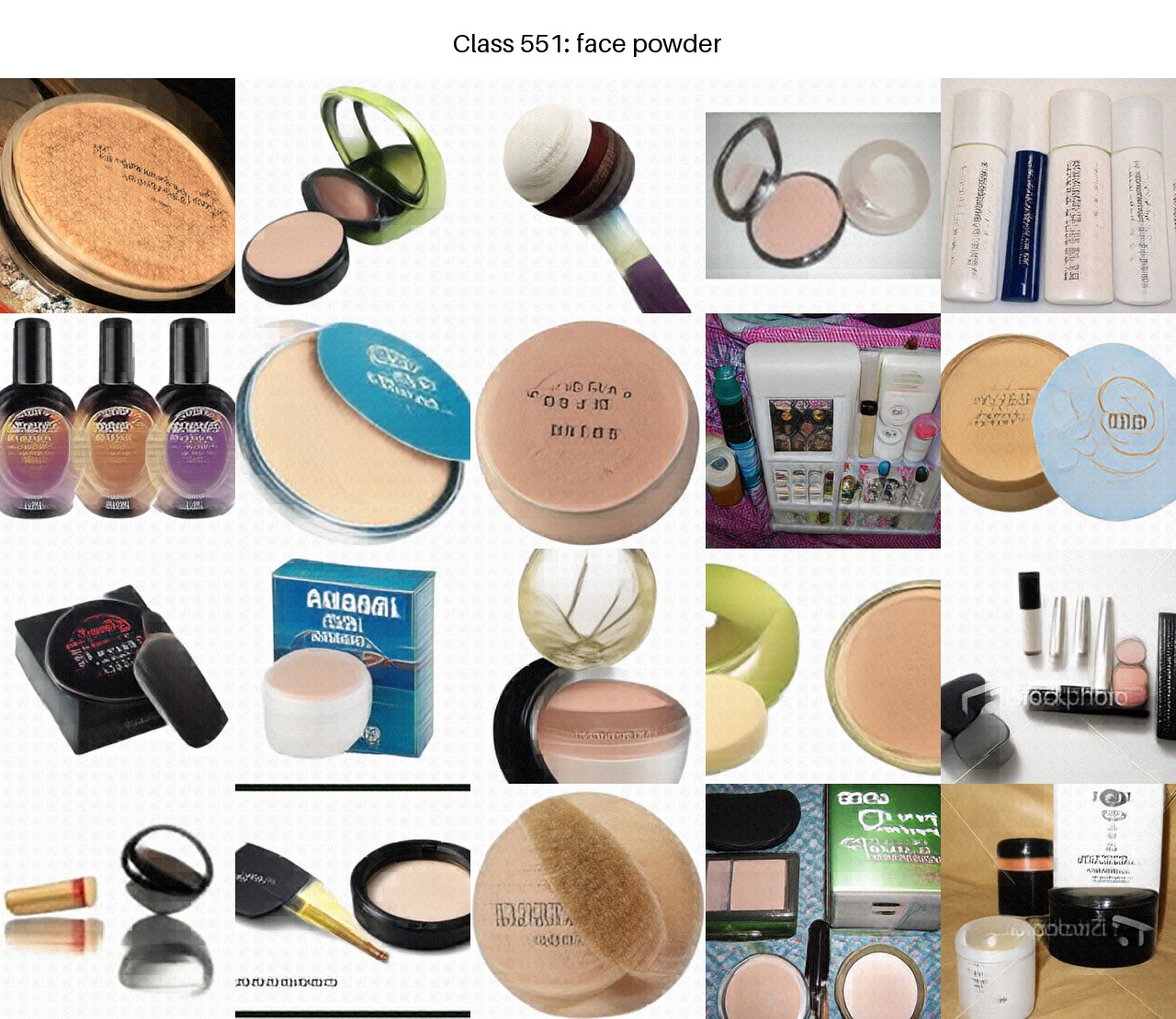}
\caption{Random, uncurated one-step ImageNet samples, classes 353--551.}
\label{fig:uncurated-onestep-2}
\end{figure}

\begin{figure}[p]
\centering
\includegraphics[width=0.48\linewidth]{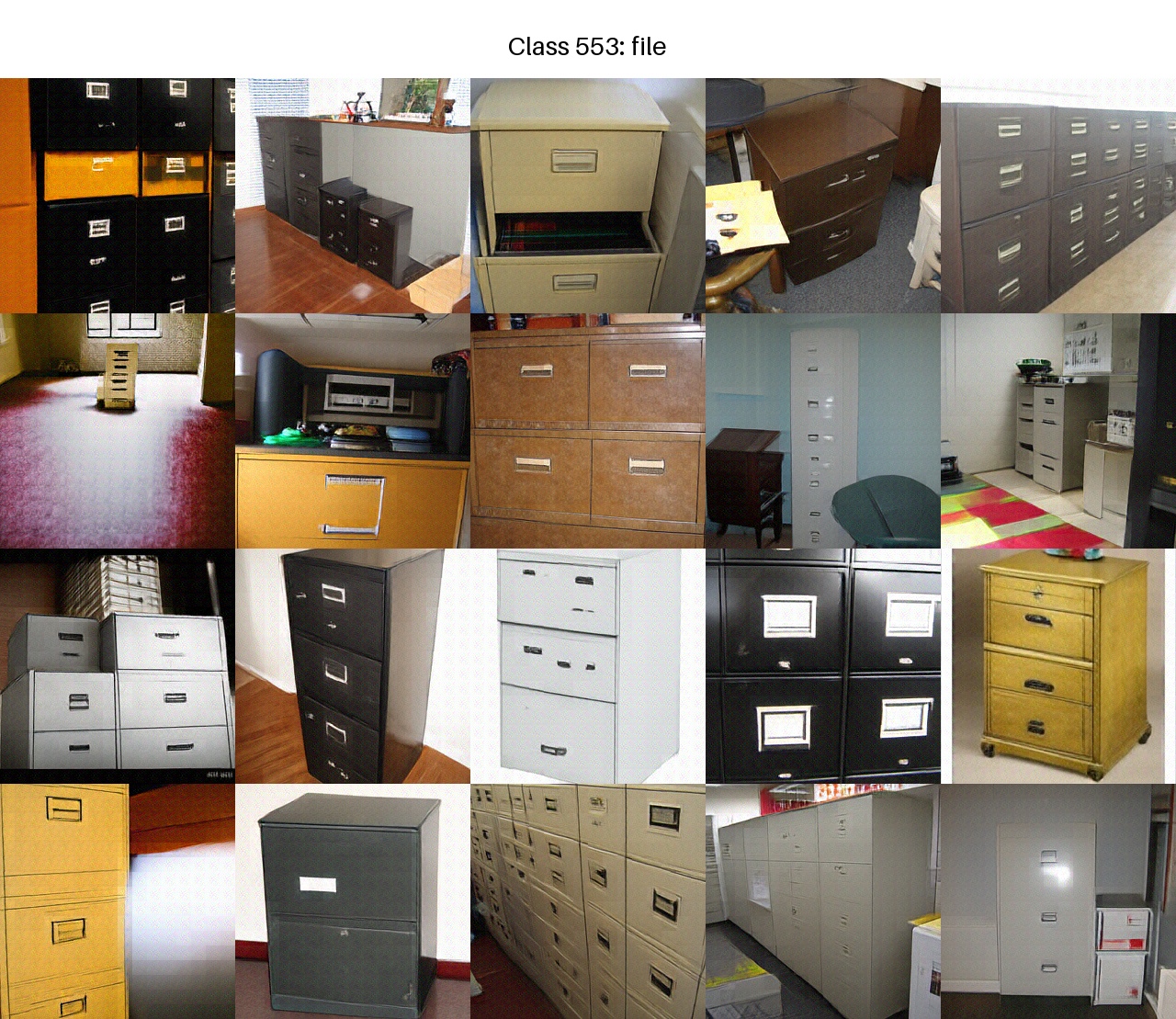}\hfill
\includegraphics[width=0.48\linewidth]{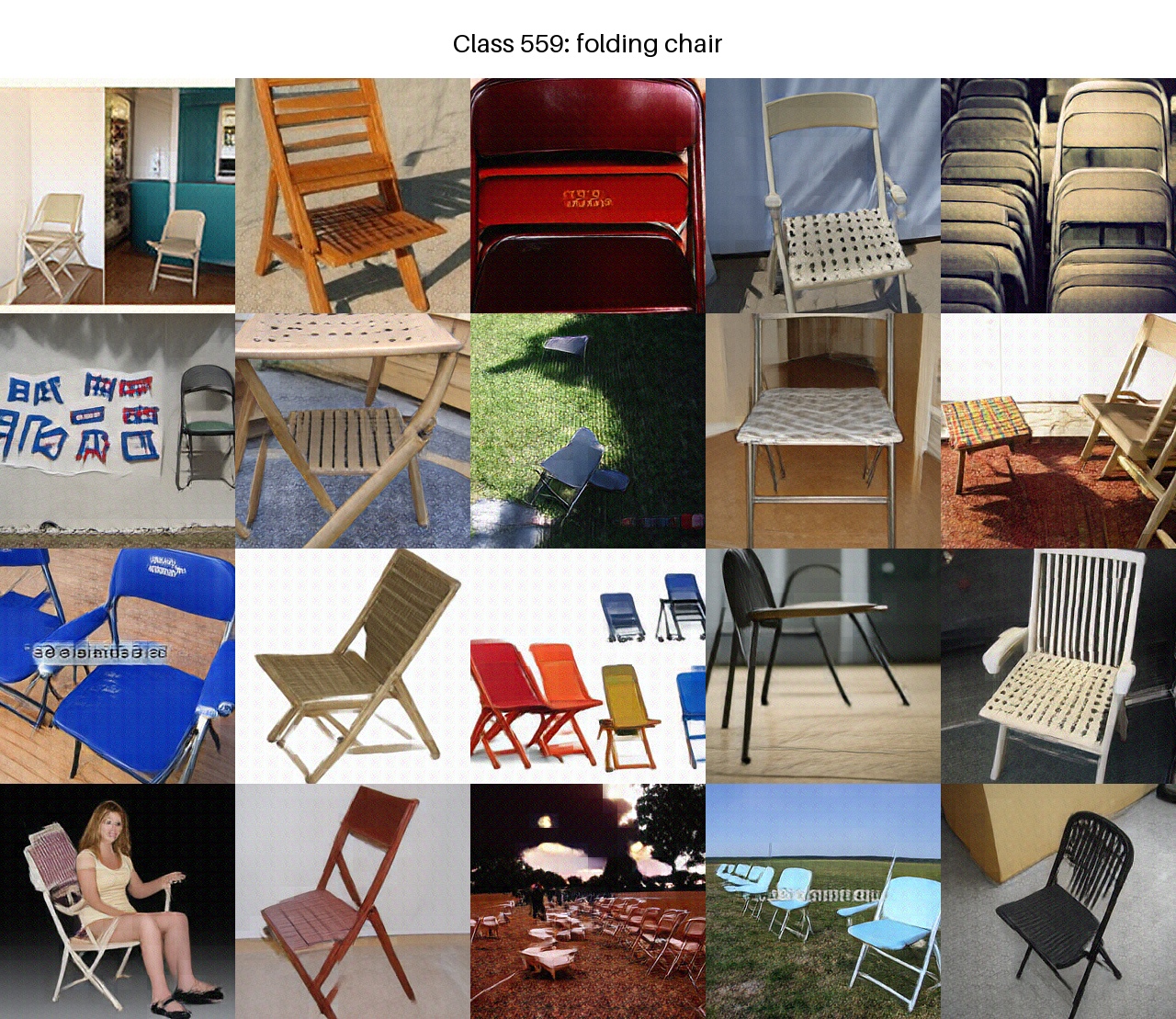}\\
\includegraphics[width=0.48\linewidth]{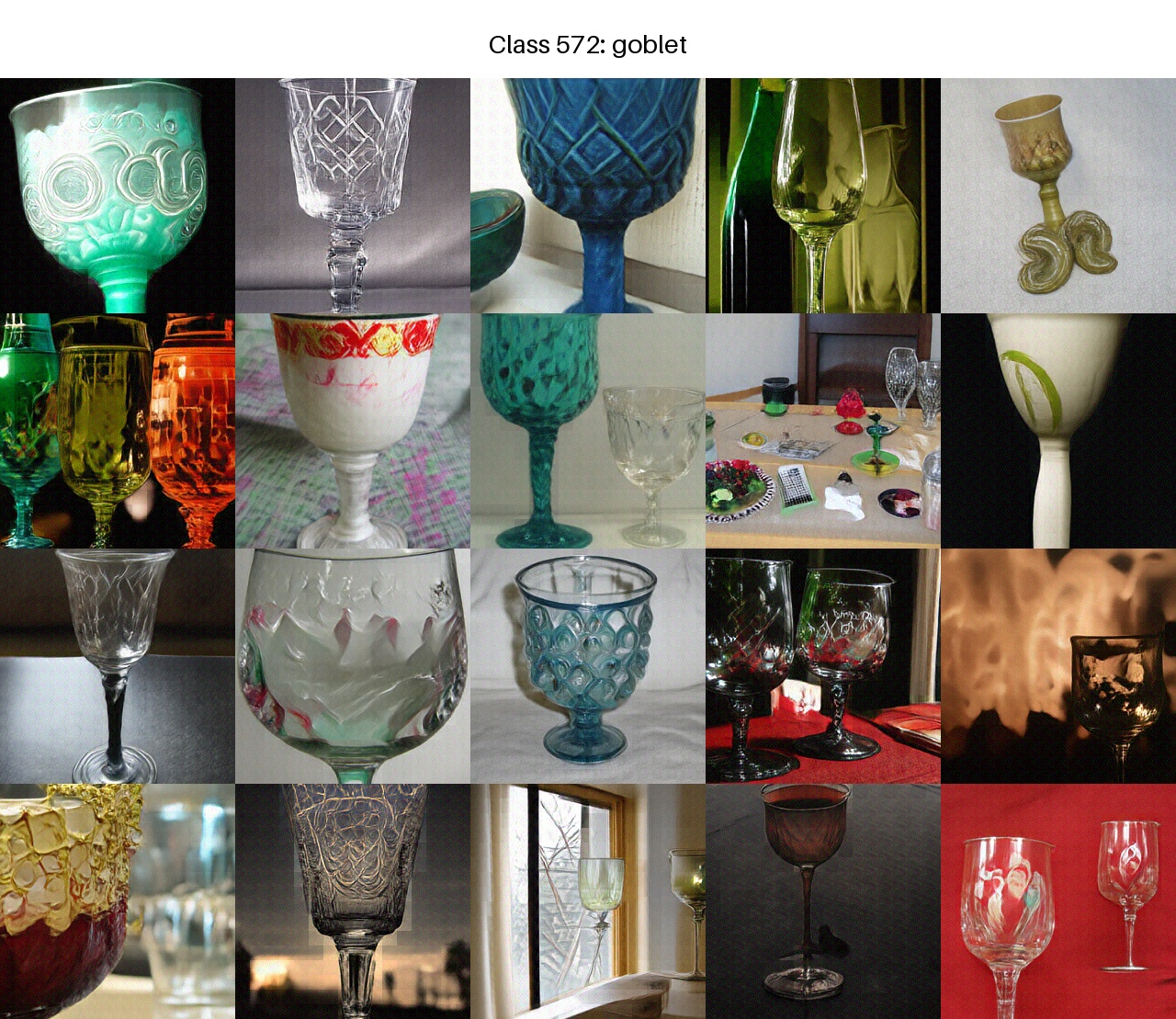}\hfill
\includegraphics[width=0.48\linewidth]{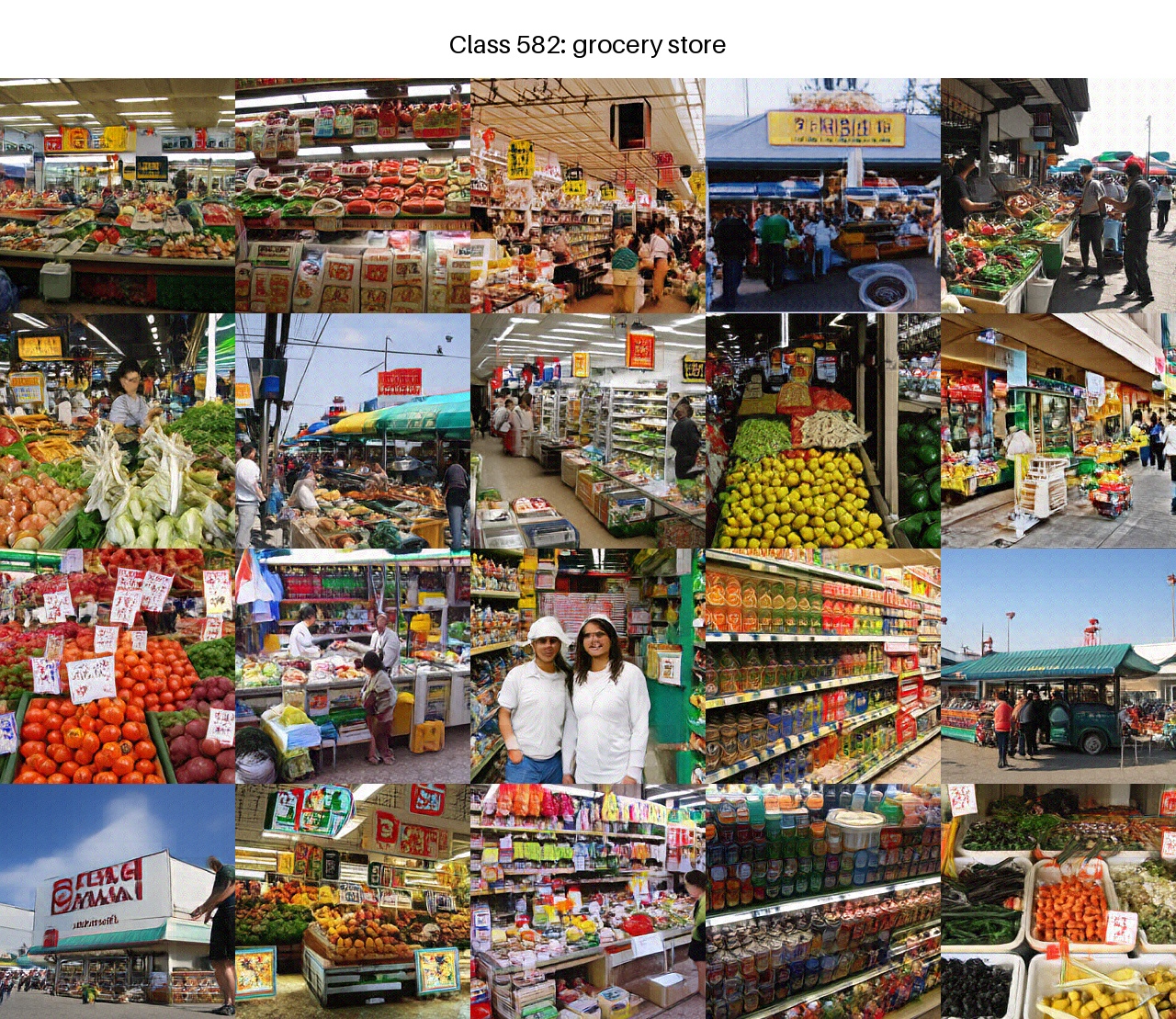}\\
\includegraphics[width=0.48\linewidth]{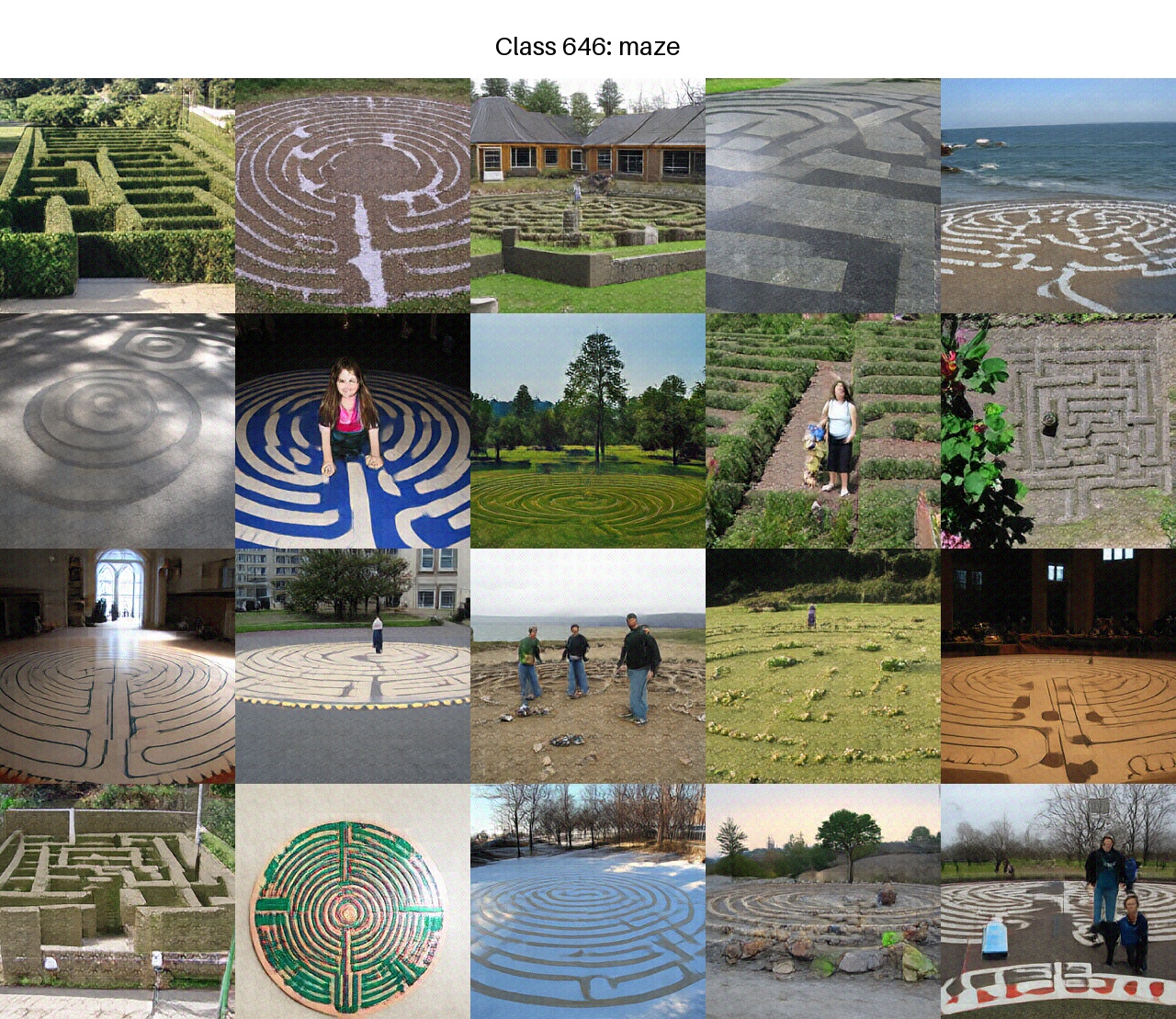}\hfill
\includegraphics[width=0.48\linewidth]{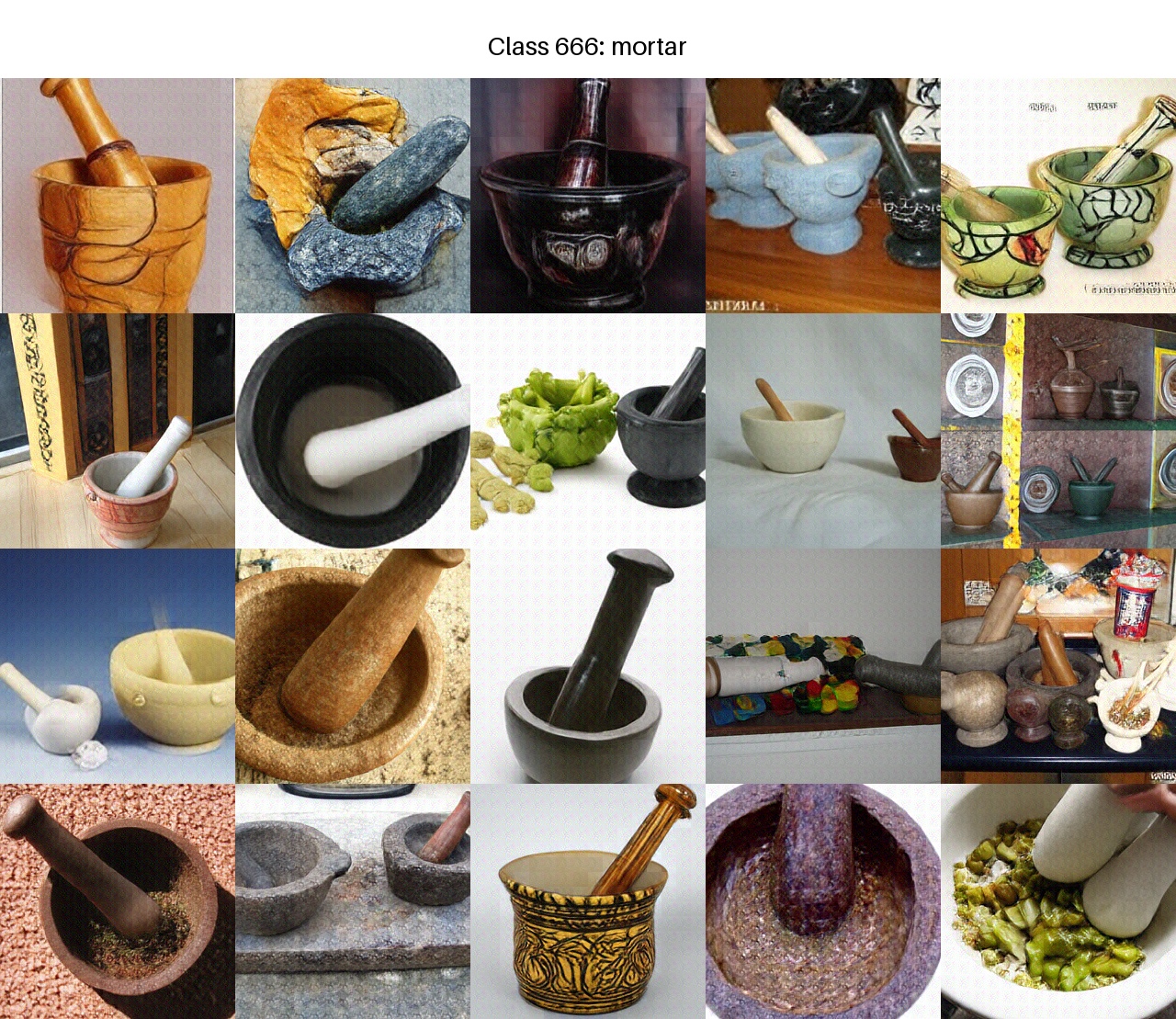}
\caption{Random, uncurated one-step ImageNet samples, classes 553--666.}
\label{fig:uncurated-onestep-3}
\end{figure}

\begin{figure}[p]
\centering
\includegraphics[width=0.48\linewidth]{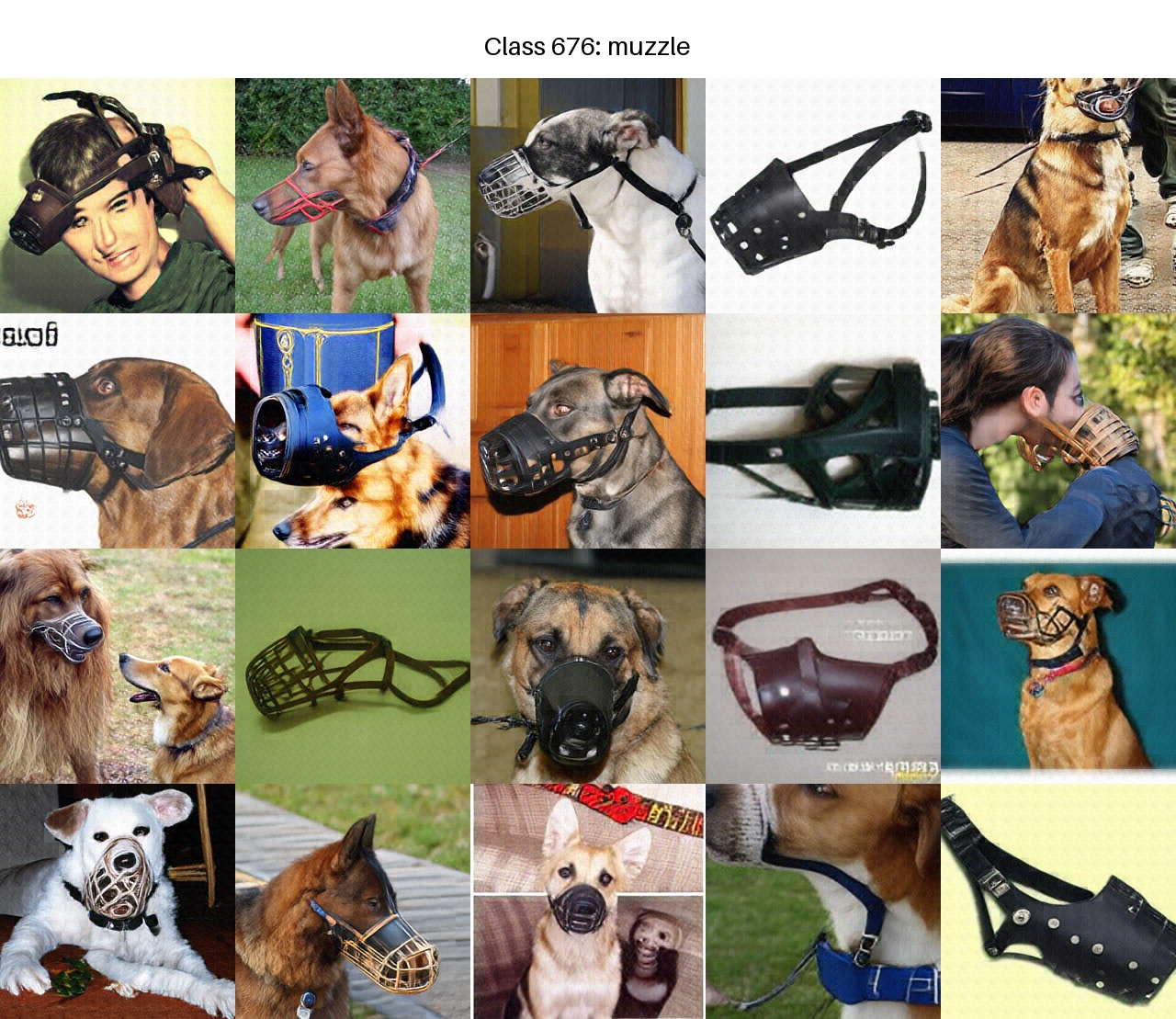}\hfill
\includegraphics[width=0.48\linewidth]{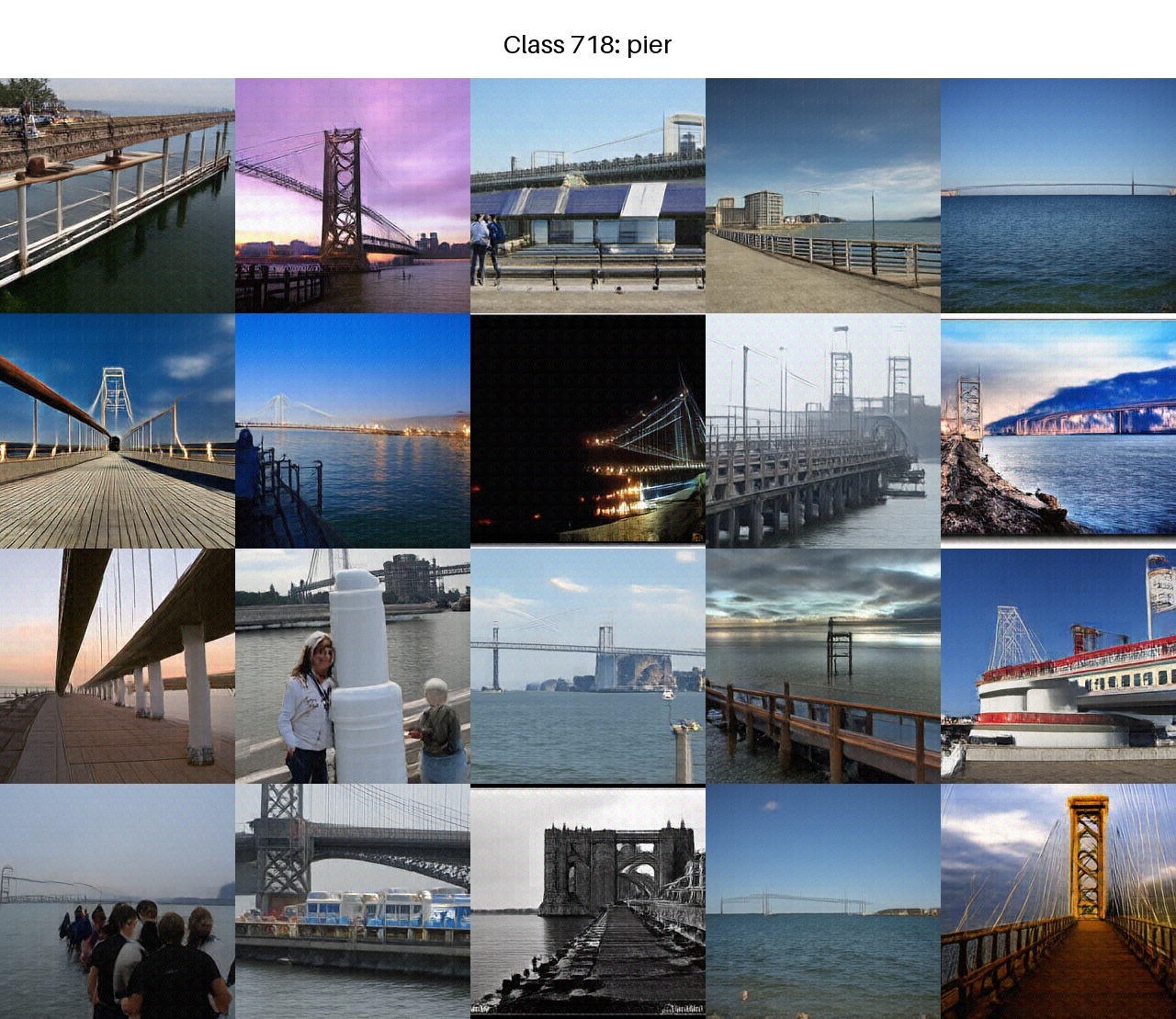}\\
\includegraphics[width=0.48\linewidth]{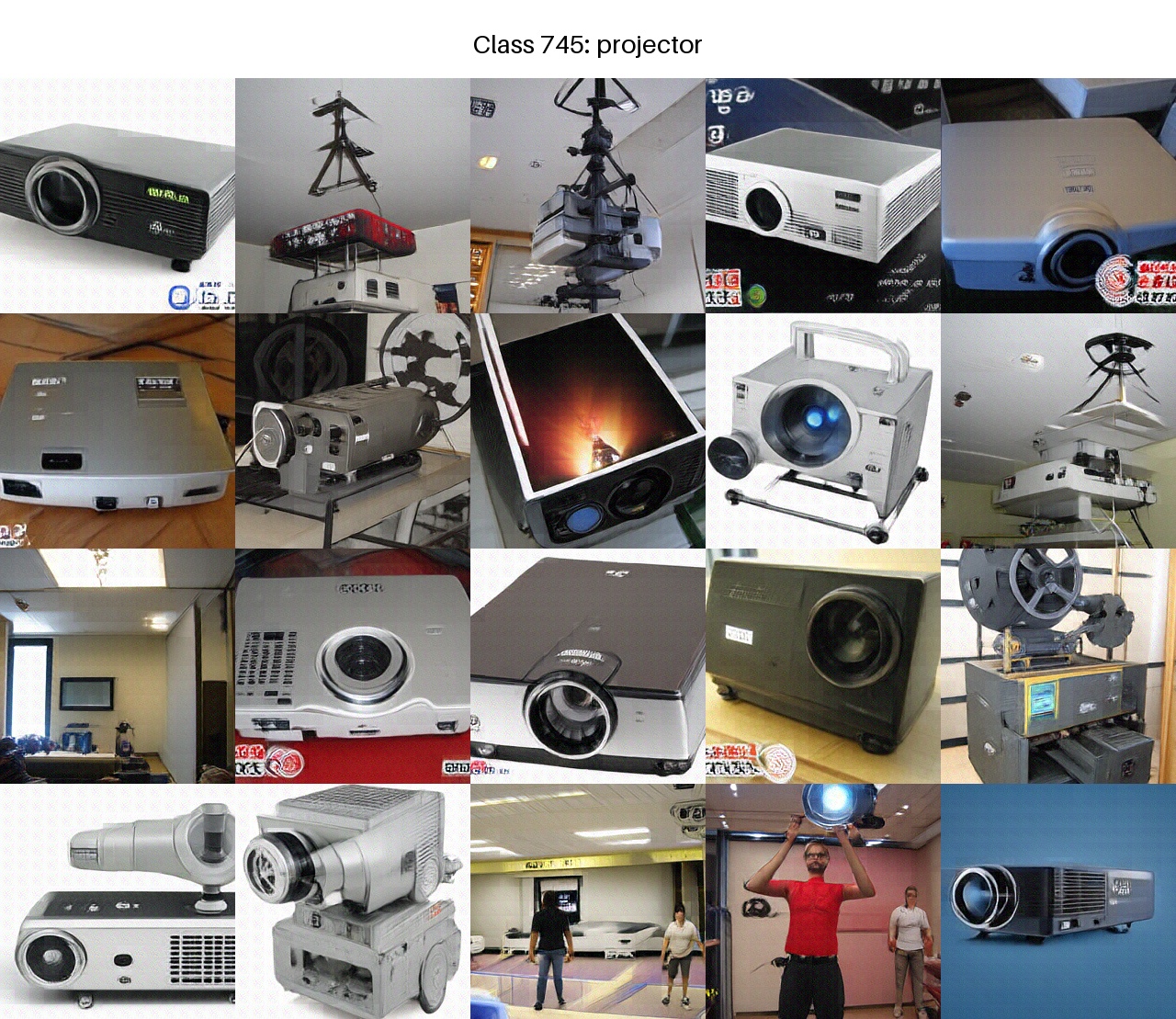}\hfill
\includegraphics[width=0.48\linewidth]{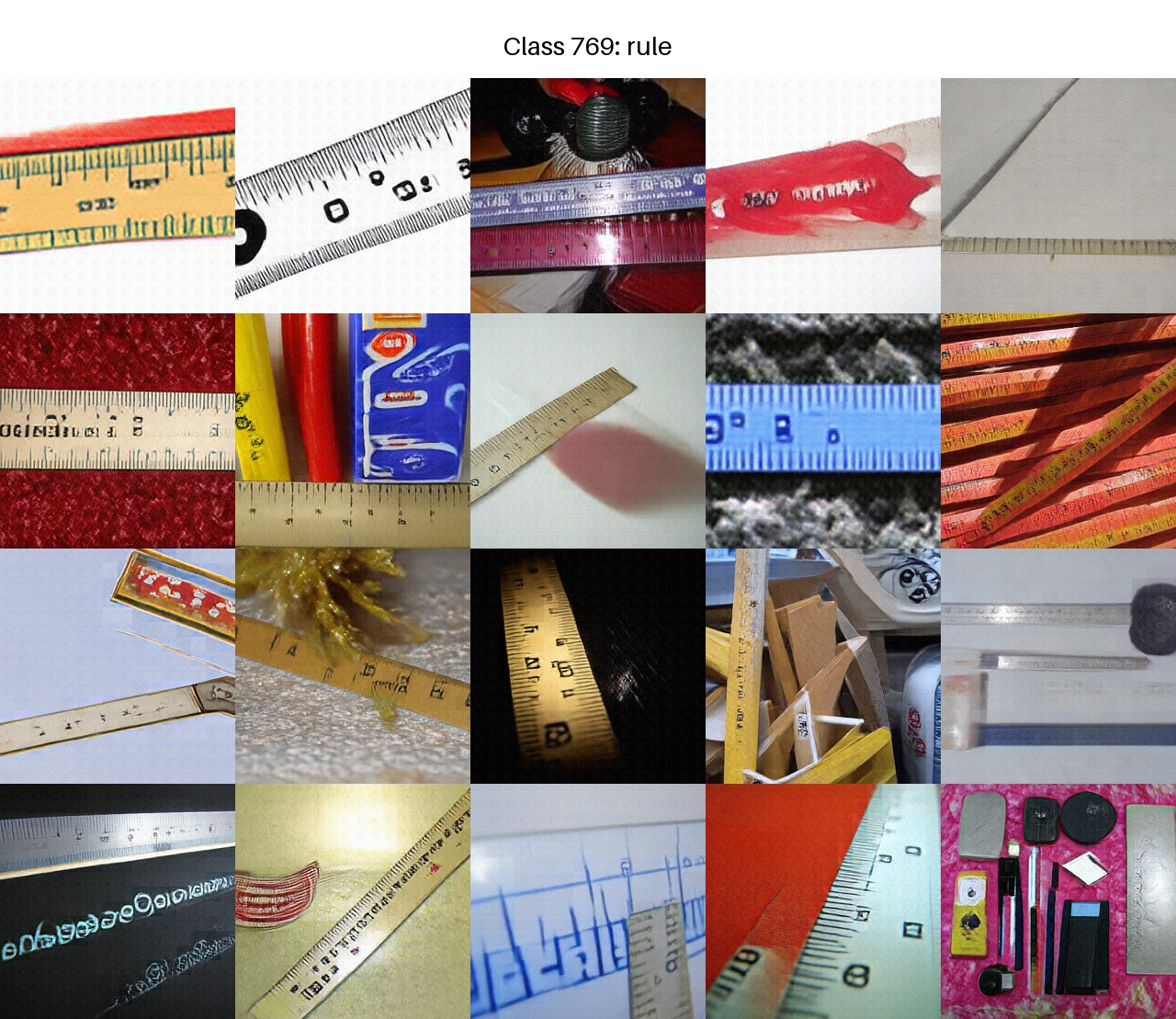}\\
\includegraphics[width=0.48\linewidth]{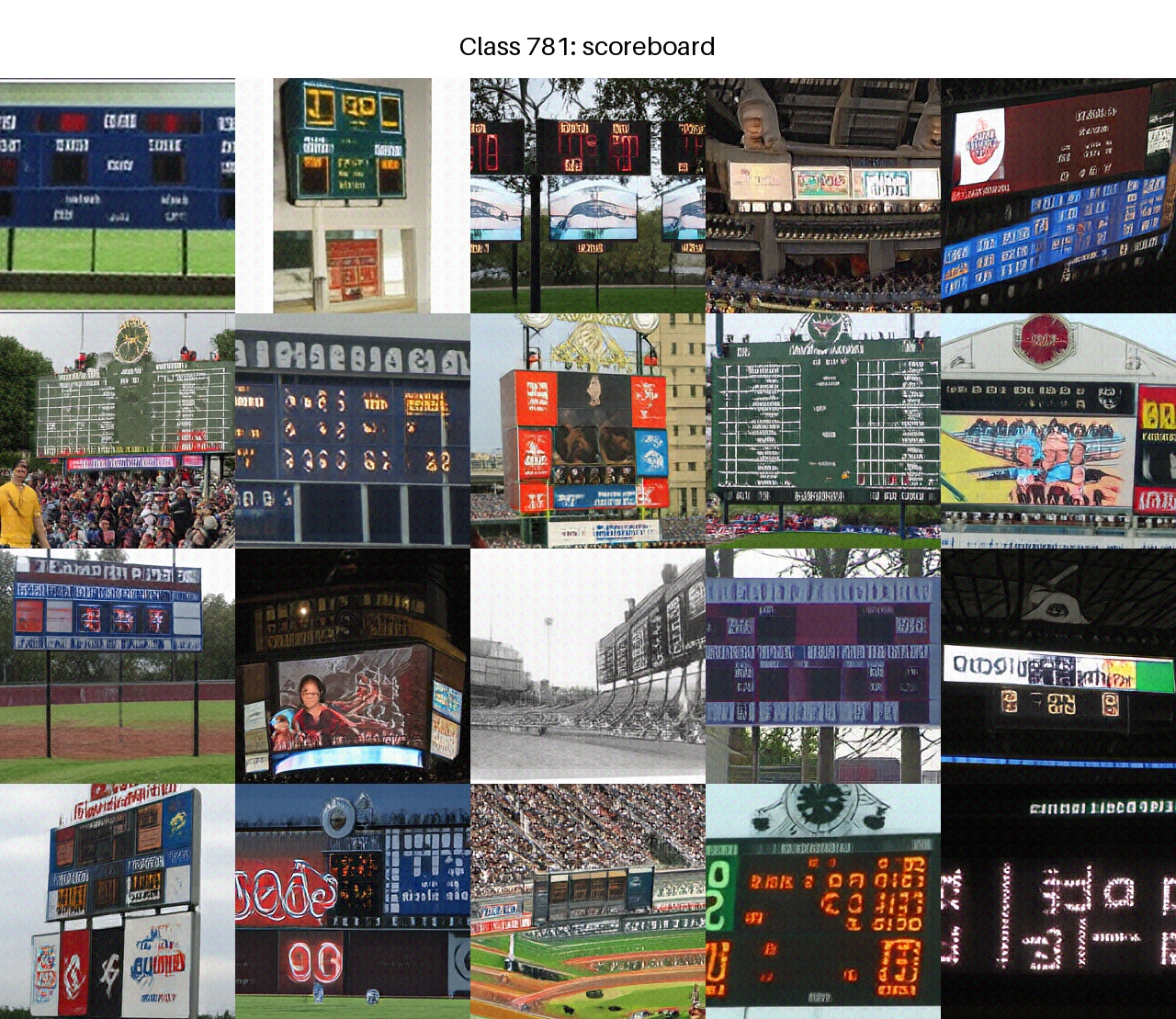}\hfill
\includegraphics[width=0.48\linewidth]{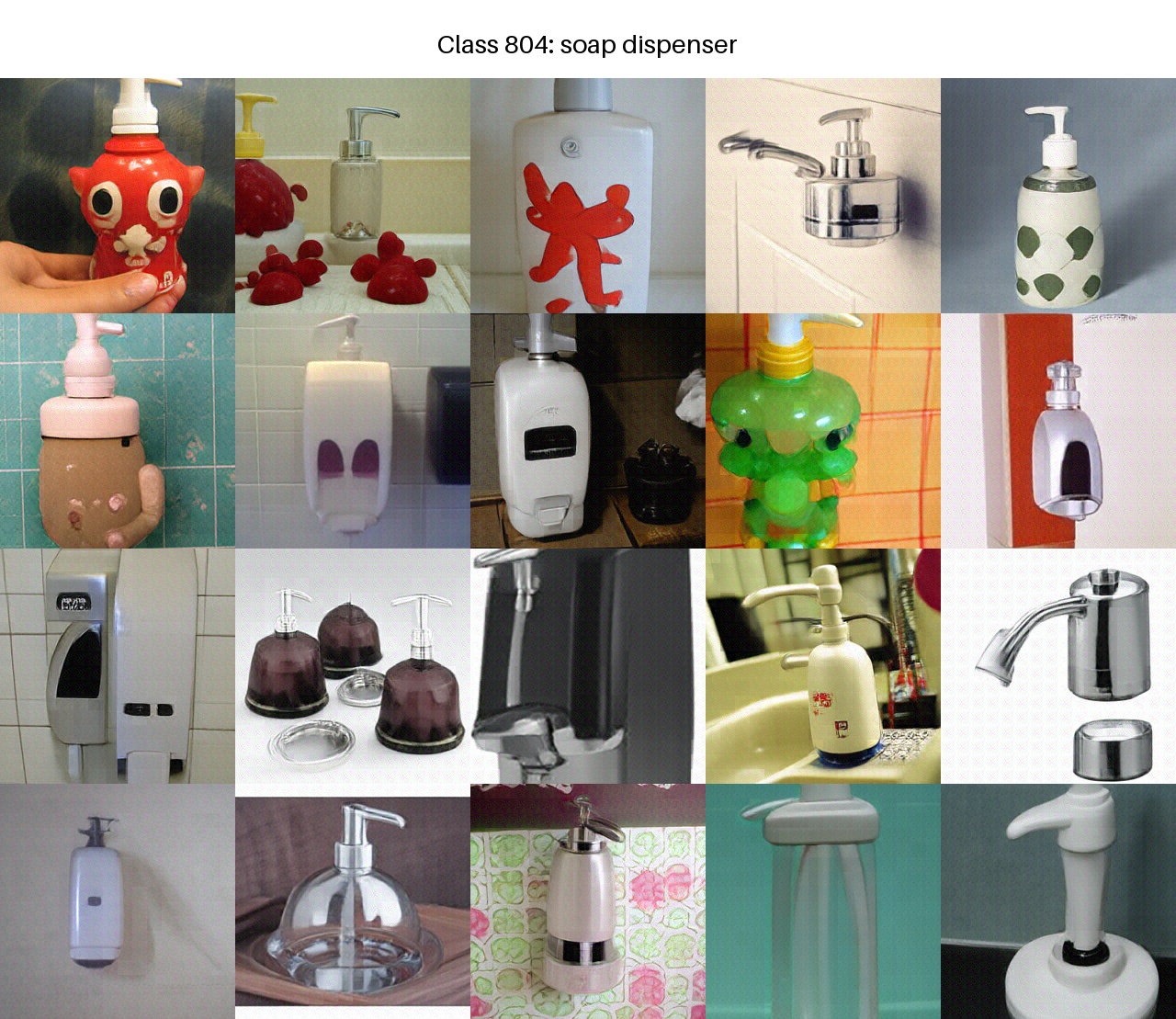}
\caption{Random, uncurated one-step ImageNet samples, classes 676--804.}
\label{fig:uncurated-onestep-4}
\end{figure}

\begin{figure}[p]
\centering
\includegraphics[width=0.48\linewidth]{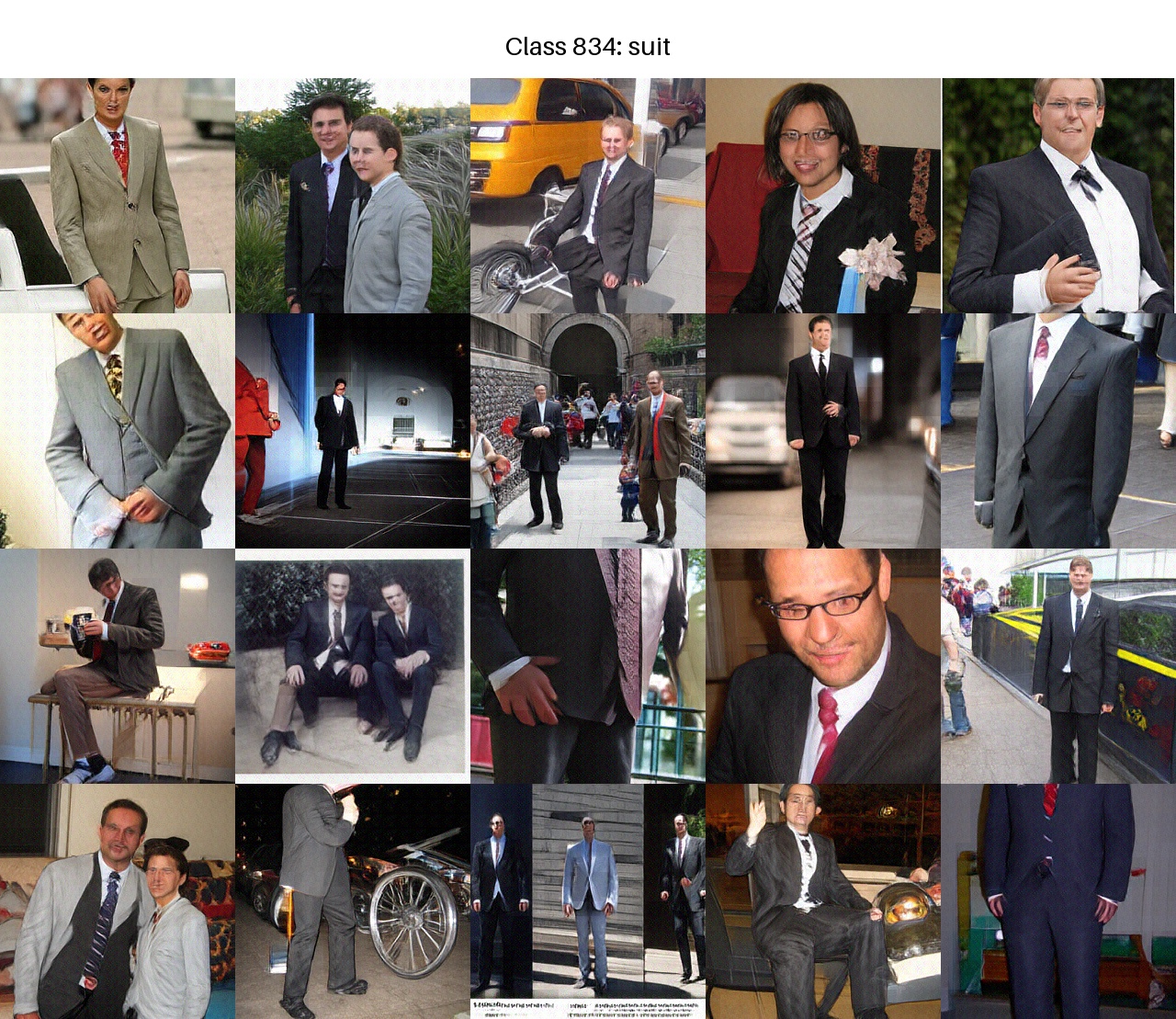}\hfill
\includegraphics[width=0.48\linewidth]{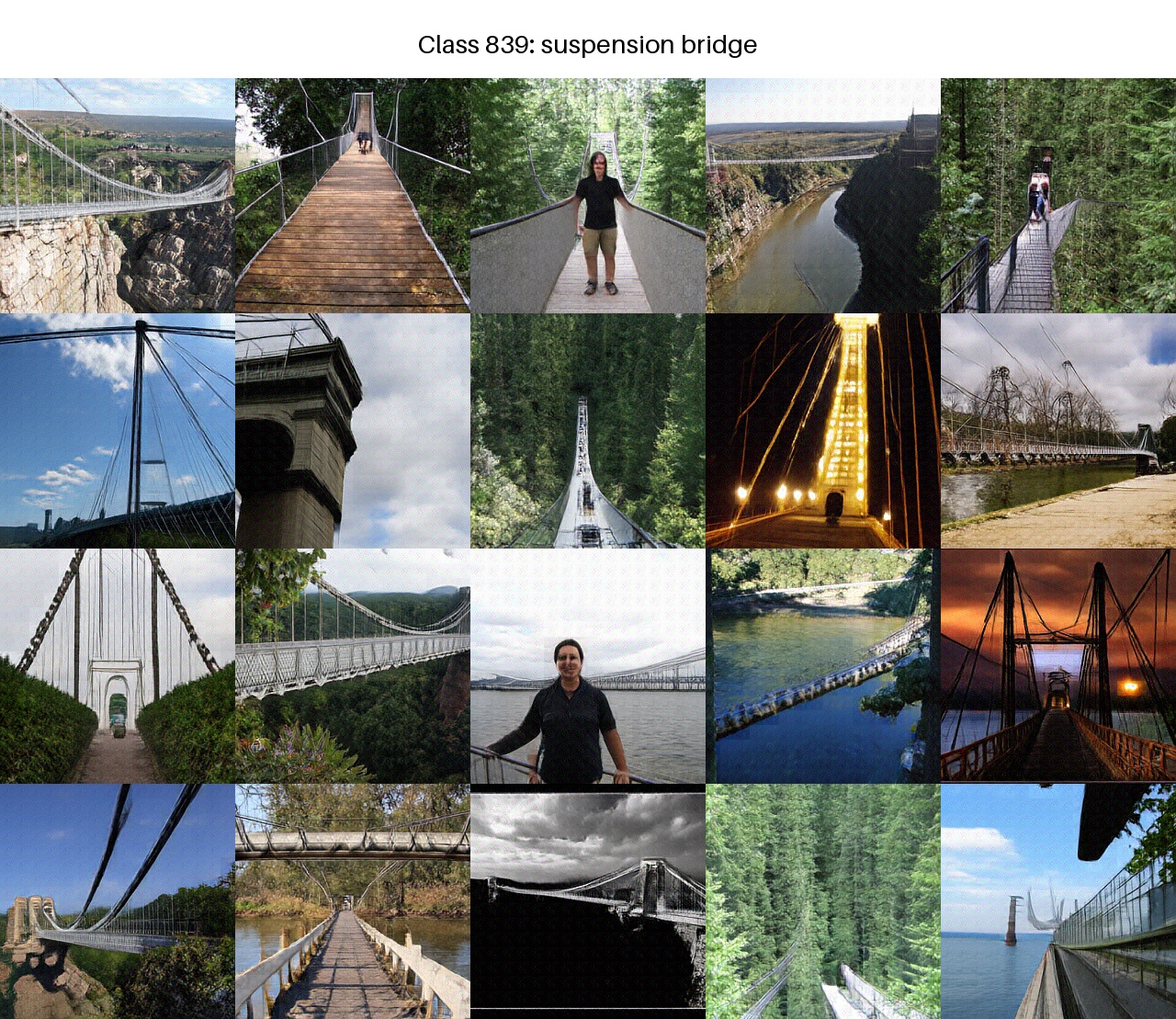}\\
\includegraphics[width=0.48\linewidth]{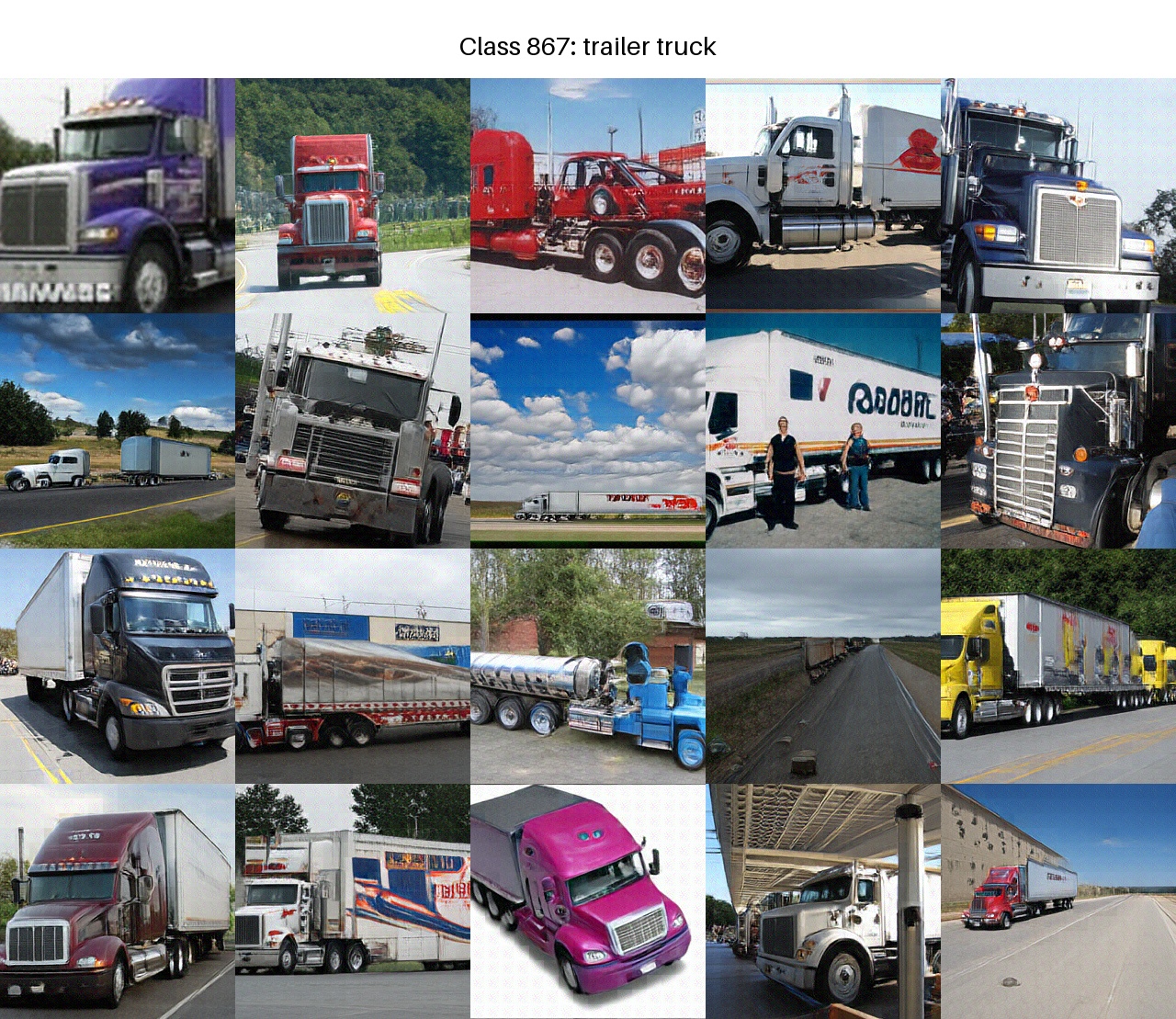}\hfill
\includegraphics[width=0.48\linewidth]{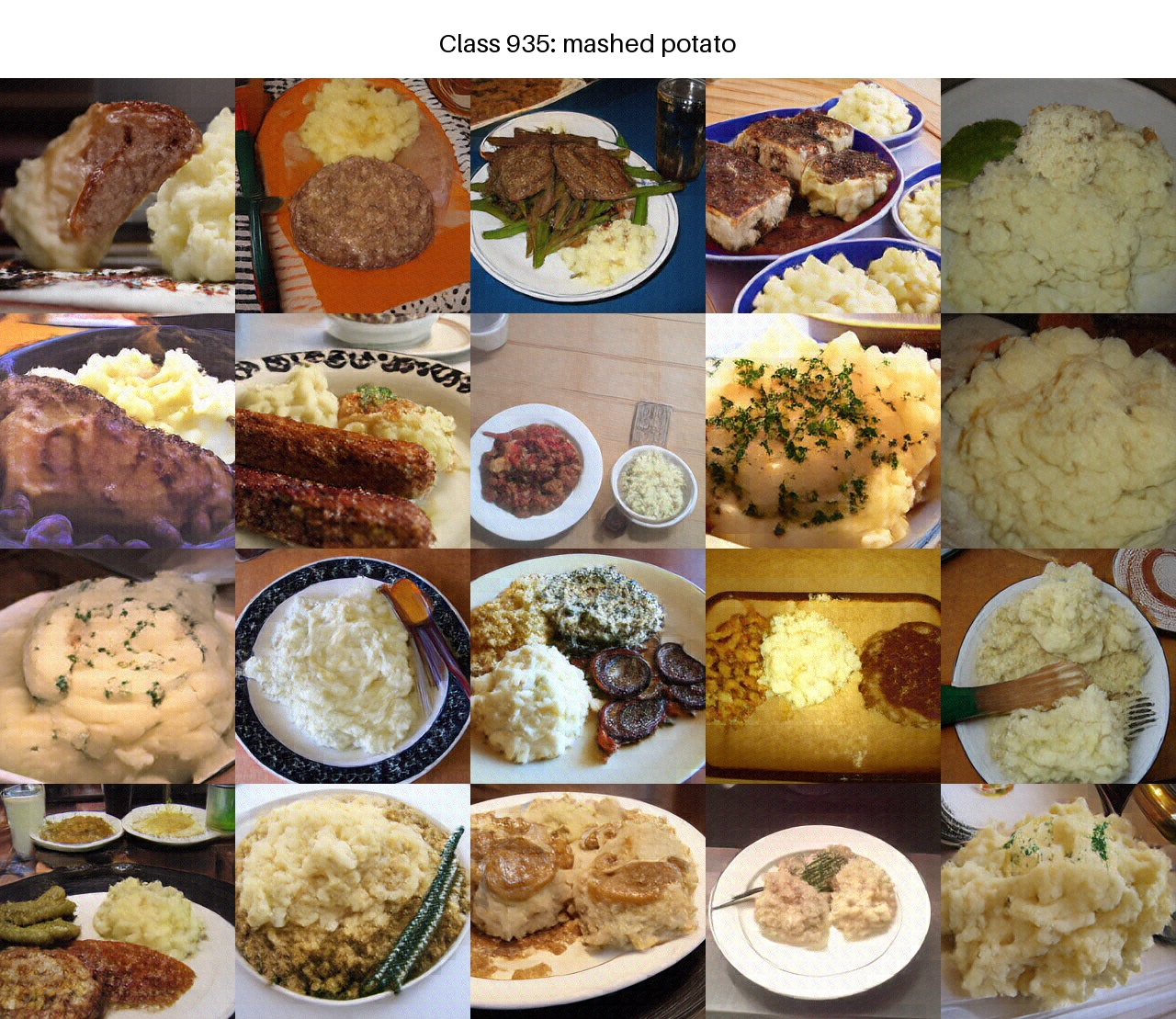}\\
\includegraphics[width=0.48\linewidth]{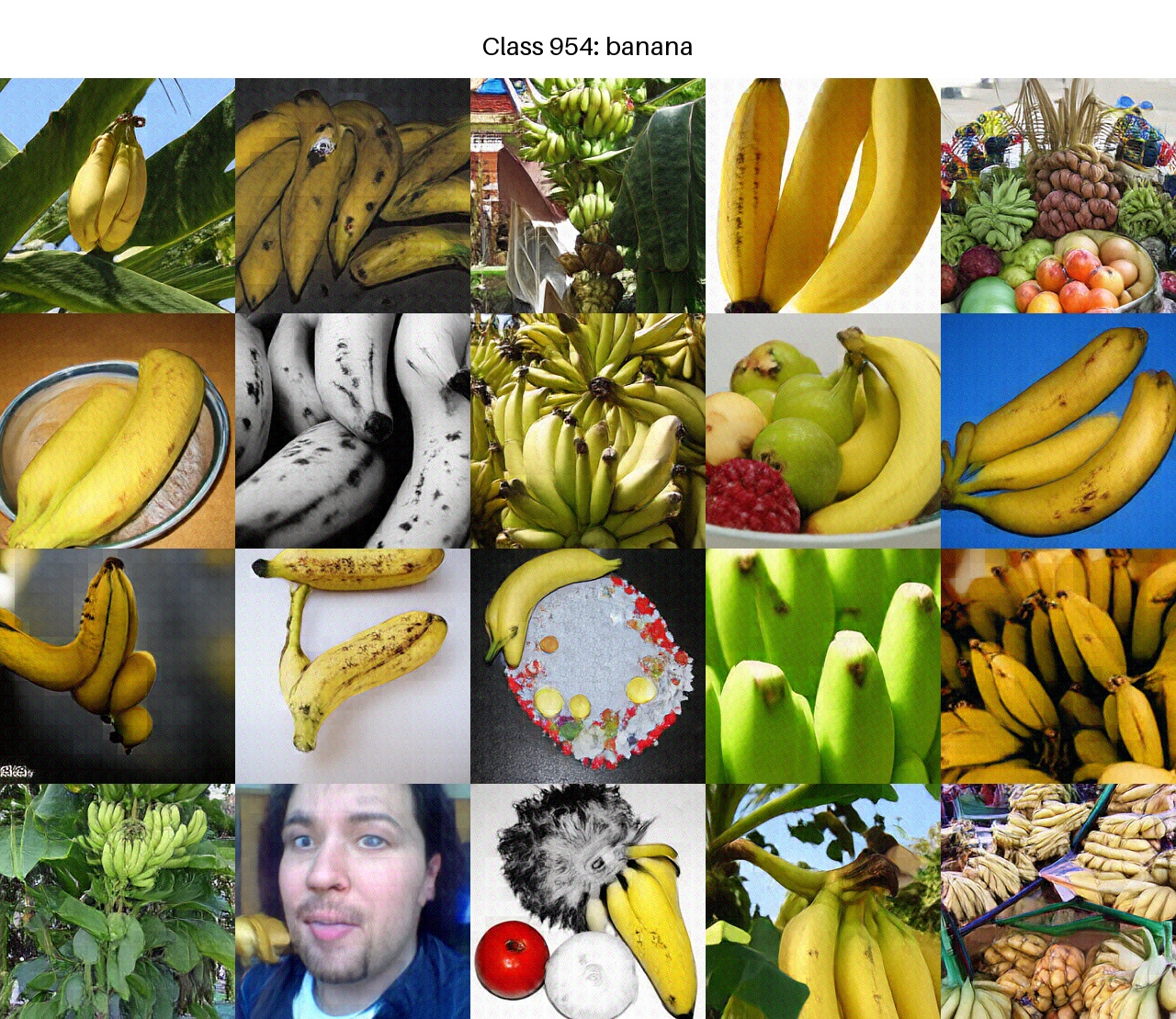}\hfill
\includegraphics[width=0.48\linewidth]{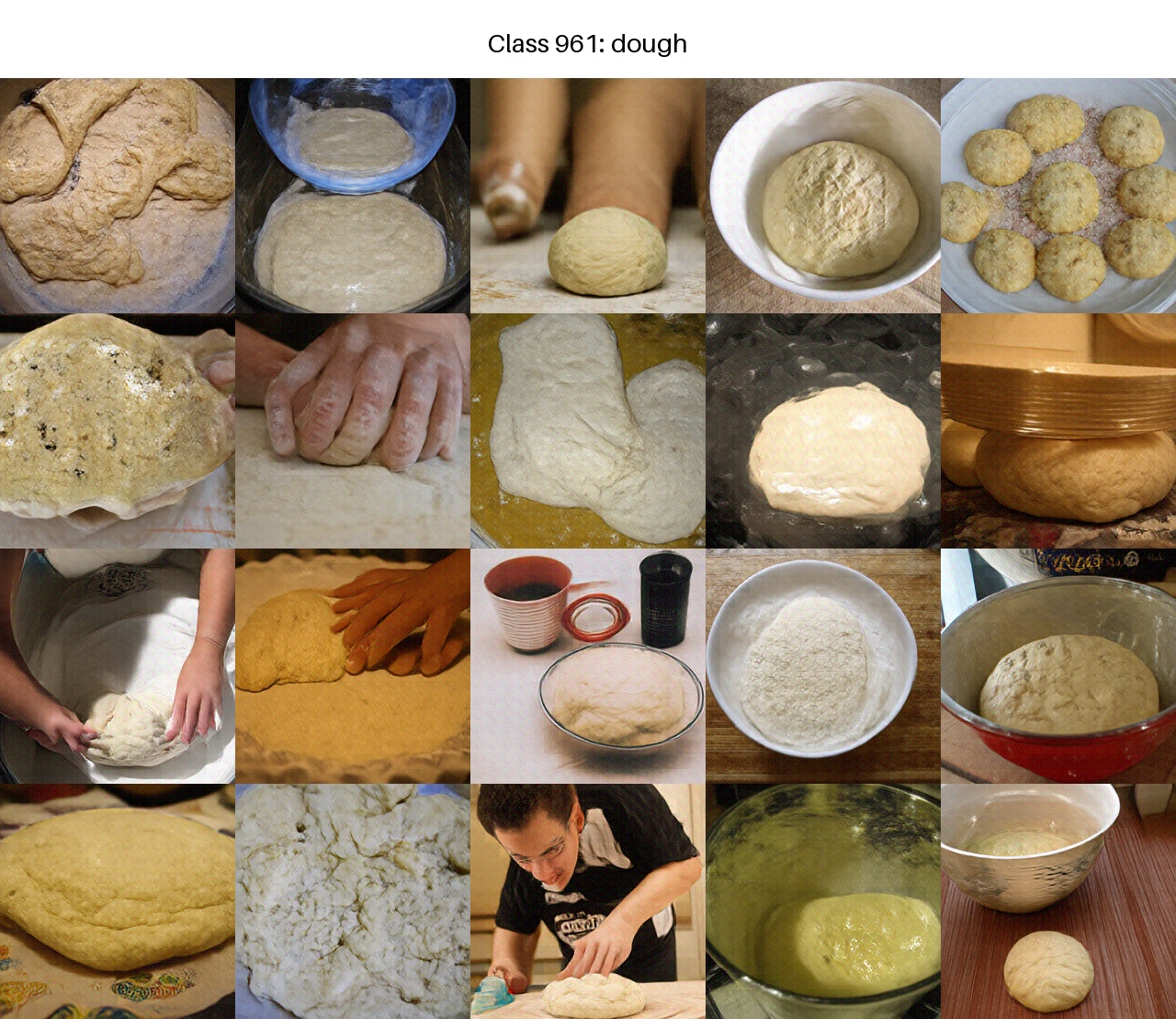}
\caption{Random, uncurated one-step ImageNet samples, classes 834--961.}
\label{fig:uncurated-onestep-5}
\end{figure}
% -----------------------------------------------------------------------------
% End inlined implementation and evaluation appendix
% -----------------------------------------------------------------------------

\end{document}